\documentclass{article}

\usepackage{microtype}
\usepackage{graphicx}
\usepackage{subfigure}
\usepackage{booktabs} 
\usepackage{multirow}
\usepackage[table,xcdraw]{xcolor}
\usepackage{hyperref}

\usepackage[preprint]{icml2026} 
\usepackage{lineno}
\usepackage{amsmath}
\usepackage{amssymb}
\usepackage{mathtools}
\usepackage{amsthm}

\usepackage[capitalize,noabbrev]{cleveref}

\theoremstyle{plain}
\newtheorem{theorem}{Theorem}[section]
\newtheorem{proposition}[theorem]{Proposition}

\theoremstyle{definition}

\theoremstyle{remark}

\usepackage[textsize=tiny]{todonotes}
\usepackage{algorithm}
\usepackage{algorithmic}

\icmltitlerunning{\textbf{TEFL:} Temporal Error Feedback Learning}

\begin{document}
\twocolumn[
\icmltitle{TEFL: Prediction-Residual-Guided Rolling Forecasting for Multi-Horizon Time Series}




\begin{icmlauthorlist}
\icmlauthor{Xiannan Huang}{tju}
\icmlauthor{Shen Fang}{zj}
\icmlauthor{Shuhan Qiu}{tju}
\icmlauthor{Chengcheng Yu}{tju}
\icmlauthor{Jiayuan Du}{tju_ee}
\icmlauthor{Chao Yang}{tju}
\end{icmlauthorlist}

\icmlaffiliation{tju}{Key Laboratory of Road and Traffic Engineering, Ministry of Education Tongji University, Shanghai China}
\icmlaffiliation{tju_ee}{The College of Computer Science, Tongji University, Shanghai China}
\icmlaffiliation{zj}{Research Center for Scientific Data Hub, Zhejiang Lab, Hangzhou, China}

\icmlcorrespondingauthor{Chao Yang}{tongjiyc@tongji.edu.cn}

\icmlkeywords{Machine Learning, ICML}

\vskip 0.3in
]



\printAffiliationsAndNotice{\icmlEqualContribution} 

\begin{abstract}

Time series forecasting plays a critical role in domains such as transportation, energy, and meteorology. Despite their success, modern deep forecasting models are typically trained to minimize point-wise prediction loss without leveraging the rich information contained in past prediction residuals from rolling forecasts—residuals that reflect persistent biases, unmodeled patterns, or evolving dynamics. We propose \textbf{TEFL} (Temporal Error Feedback Learning), a unified learning framework that explicitly incorporates these historical residuals into the forecasting pipeline during both training and evaluation. To make this practical in deep multi-step settings, we address three key challenges: (1) selecting observable multi-step residuals under the partial observability of rolling forecasts, (2) integrating them through a lightweight low-rank adapter to preserve efficiency and prevent overfitting, and (3) designing a two-stage training procedure that jointly optimizes the base forecaster and error module. Extensive experiments across 10 real-world datasets and 5 backbone architectures show that TEFL consistently improves accuracy, reducing MAE by 5–10\% on average. Moreover, it demonstrates strong robustness under abrupt changes and distribution shifts, with error reductions exceeding 10\% (up to 19.5\%) in challenging scenarios. By embedding residual-based feedback directly into the learning process, TEFL offers a simple, general, and effective enhancement to modern deep forecasting systems.
\end{abstract}

\section{Introduction}
Time series forecasting is of great significance in many fields such as finance, transportation, healthcare, meteorology, and energy \cite{gruca2023weather4cast,kadiyala2014vector,morid2023time}. In recent years, deep neural networks have become dominant in this domain \cite{liu2023itransformer,zhou2021informer}. These models are typically trained on randomly sampled consecutive segments: the past segment serves as input, and the future segment as target. 

However, this training protocol creates a critical mismatch with real-world deployment. In practice, forecasting systems operate in a \textit{rolling} manner: at each time step, predictions are made for the future, and ground truth for the current step becomes available shortly after. This sequential revelation of outcomes yields a stream of \textit{past prediction residuals}—a rich, causal signal that reflects the model’s recent performance and systematic biases.

Despite its availability and relevance, this residual signal is \textit{not used during training} in standard deep forecasting pipelines. Models are optimized to minimize point-wise loss on isolated segments, with no mechanism to learn how to interpret or leverage their own past errors for future correction. Consequently, even when residuals clearly indicate persistent over- or under-prediction, the model remains "blind" to this feedback during inference.

As a result, we propose TEFL (Temporal Error Feedback Learning), a general framework that explicitly incorporates past prediction residuals from rolling forecasts into the forecasting process—during both training and evaluation. Formally, the forecasts for time step $t$ are obtained as:
\begin{equation}
\hat{y}_t = f([y_{t-1}, \dots, y_{t-p}]) + g([\epsilon_{t-1}, \dots, \epsilon_{t-q}]),
\end{equation}
In this formulation, $y_t$ is the time series value at time $t$ , $p$ and $q$ denote the lengths of the observation window and the error feedback window, respectively. The functions $f$ and $g$ are parameterized as neural networks (e.g., Transformer or MLP).

Despite its simplicity, implementing this idea in modern deep forecasting faces three challenges:
\begin{itemize}
   \item \textbf{Error Selection}: In multi-step forecasting, the prediction error at each time step is a vector of length $k$, but only a prefix of it is observable at next time. We define a selection rule that identifies which errors can be reliably used as feedback signals, ensuring temporal consistency and full observability.
   \item \textbf{Integration Design}: To avoid interfering with the base forecaster, we design a decoupled correction module that adds a small adjustment signal—computed solely from historical errors—to the base prediction. This design ensures minimal computational overhead and compatibility with different model architectures.
   \item \textbf{Training the base model and the error-integration module}: 
Joint training is unstable because the error module relies on predictions from the base model, which may be inaccurate early in training. We address this with a two-stage strategy: first warm up the base model to produce reliable errors, then jointly fine-tune both components.
\end{itemize}

As a result, we propose \textbf{TEFL} (Temporal Error Feedback Learning), a general and lightweight framework that enables deep forecasters to learn from their own historical residuals. This design provides consistent improvements across diverse forecasting scenarios:
\begin{itemize}
    \item \textbf{Improved baseline accuracy}: On standard benchmark dataset, residual feedback corrects systematic biases, reducing MAE by 5–10\%.
    \item \textbf{Adaptation under distribution shifts}: When the data distribution changes (e.g., seasonal transitions), recent residuals automatically adjust forecasts without retraining.
    \item \textbf{Robustness to abrupt shocks}: Large errors from anomalies are directly propagated to subsequent predictions, mitigating error propagation.
\end{itemize}

We summarize the contributions of our work as follows.

\begin{enumerate}
\item We identify and formalize the training-deployment mismatch caused by ignoring rolling-residual feedback in deep time series forecasting.
\item We propose TEFL, a general framework with principled solutions to error observability, architectural integration, and training stability.
\item We demonstrate TEFL’s effectiveness and compatibility across a wide range of models and datasets.
\end{enumerate}
\section{Related work}
\subsection{Deep learning based time series prediction model}
In recent years, deep learning has become the dominant approach in time series forecasting, with various model architectures proposed to capture complex dependencies in sequential data. Transformer-based models, such as Informer \cite{zhou2021informer}, PatchTST \cite{nietime}, and iTransformer \cite{liu2023itransformer}, excel at modeling long-term temporal dependencies through attention mechanisms. They further improve efficiency and performance with techniques like sparse attention, series patching, and channel-independent embeddings \cite{zhou2021informer}. Meanwhile, CNN-based models (e.g., TimesNet) \cite{wu2022timesnet,huang2024leveraging} are effective at extracting local and periodic patterns, while MLP-based models (e.g., DLinear \cite{zeng2023transformers}, TimeMixer \cite{wang2024timemixer}) offer lightweight and efficient solutions with multi-scale processing. 

Regarding training strategies, traditional methods typically compute regression losses (e.g., MSE/MAE) uniformly over the whole sequence. Recently, researchers have started to explore more refined training paradigms. For example, curriculum learning \cite{wu2020connecting} gradually increases the forecasting horizon to adjust the training difficulty, while selective learning \cite{fu2025selective} goes further by identifying and masking time points in the target sequence that are less generalizable. Other work replaces the standard loss with alternatives such as Dynamic Time Warping \cite{cuturi2017soft}, frequency-domain loss \cite{wang2024fredf}, or distribution matching \cite{kudrat2025patch}, aiming to better align predicted and target sequences in terms of temporal or spectral shapes. However, a common limitation across these approaches is that they treat each training segment in isolation, without leveraging the natural feedback loop present in real-world rolling forecasting—where past prediction residuals become available sequentially and could inform future predictions. As a result, even when systematic biases emerge during evaluation, the model has no mechanism to learn from or correct them.

\subsection{Forecasting with historical prediction residuals}
A line of prior work has explored using historical prediction residuals to improve forecasting accuracy. While conceptually related to our approach, they differ fundamentally in assumptions, training protocol, and applicability to modern deep forecasting.

\textbf{(1) Two-stage residual modeling with shallow models}  
Early hybrid approaches follow a two-stage pipeline: first, a simple linear forecaster—such as an AR or ARIMA model—is fitted to the time series; then, a machine learning model (e.g., neural networks \cite{zhang2003time}, support vector machines \cite{nie2012hybrid}, or ensemble methods \cite{khashei2011novel}) is trained to model the residuals from the linear model. This decoupled training is effective only when the base model is low-capacity and not prone to overfitting, ensuring that training residuals reflect true error patterns. In contrast, when the base model is a high-capacity deep network, training-set residuals become unreliable due to overfitting, leading to poor generalization of the residual model. Moreover, these methods cannot handle multi-step forecasting, where only a prefix of past errors is observable. This limitation highlights the need for joint training and causally consistent error selection.

\textbf{(2) Statistical modeling under autocorrelated errors}  
In the statistics literature, several works study model estimation under autocorrelated errors—a setting where standard estimation assumptions break down. For example, \citet{xiao2003more} proposed a pre-whitening strategy; \citet{chen2015varying} studied varying-coefficient models under error autocorrelation; \citet{lei2016estimation} used B-splines with Whittle likelihood for joint  estimation; \citet{oliveira2021additive} combined penalized splines with autoregressive structures; and \citet{liu2022bspline} introduced B-spline-based generalized least squares for correlated errors. While some employ iterative or two-stage estimation strategies that loosely inspire our warm-up + fine-tuning design, they focus on parameter estimation rather than predictive performance, assume smooth functional forms, and do not support deep architectures or rolling multi-horizon forecasting.

\textbf{(3) Classical error-feedback mechanisms in time series prediction}
The use of past prediction errors for correction is not new—it lies at the heart of several foundational frameworks in statistical time series analysis. In the Box–Jenkins paradigm, significant autocorrelation in AR model residuals motivates the inclusion of moving-average terms, yielding ARMA models whose optimal predictors explicitly incorporate lagged residuals \citep{bartholomew1971time}. Similarly, the Kalman filter in linear state-space models updates its state estimate by feeding back the one-step-ahead innovation (i.e., prediction error) scaled by an optimal gain derived from the system’s noise structure \citep{durbin2012time}. These approaches provide principled, theoretically optimal solutions—but only under strong parametric assumptions (linearity, Gaussianity, known dynamics). Crucially, error feedback is embedded within the model specification, making it non-modular and inapplicable to black-box or deep forecasters. Our work revisits this classical insight as a universal post-hoc correction, decoupling the feedback mechanism from model internals while preserving its core benefit.
\section{Method}

\subsection{Selection of historical prediction errors}

In rolling multi-step forecasting, the model predicts a full horizon (e.g., 96 steps) at each time step. However, only the immediate next observation becomes available in real time. This means that, at time  $ t $ , the only prediction errors available for feedback are those from forecasts made at or before time  $ t - H $ , where  $ H $  is the prediction horizon—because only then has the full ground truth become observable.

To respect this causality constraint during training, we must select historical errors that would be genuinely available in a real rolling deployment. As illustrated in Figure \ref{fig:err}, consider a case where the model forecasts 4 steps ahead. At time $t=8$ , only the prediction errors marked by the triangular region have already been observed—and therefore, only those errors can be used for training.

Within this feasible set, our method adopts a simple yet effective strategy: we use the complete error vector from exactly one horizon ago, i.e., the errors produced at time  $ t - H $, whose predictions cover steps  $ t - H + 1 $  through  $ t $. This choice aligns with the natural rhythm of rolling forecasting—after issuing a  $ H $ -step forecast, the system waits  $ H $  steps to receive full feedback before making the next prediction. By mirroring this delay, our training protocol better simulates the temporal structure of real-world deployment.

Moreover, using the full error vector across all horizons provides richer supervisory signals than selecting only a subset (e.g., the first-step error), as it captures horizon-specific biases and correlations. While alternative selections are possible, our ablation study in Appendix \ref{secction:error_define} shows that the proposed delayed-full-horizon strategy consistently yields the best performance.
\begin{figure}
    \centering
    \includegraphics[width=0.95\linewidth]{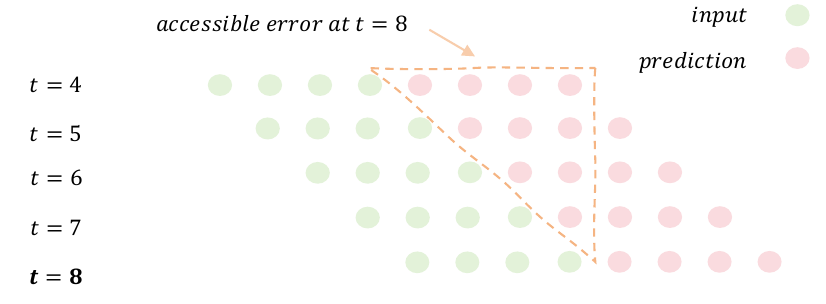}
    \caption{Example of historical errors}
    \label{fig:err}
\end{figure}
\subsection{The design of error module}

Given the selected error signals, we now turn to their integration into the forecasting process. In our framework, these errors are processed by a lightweight adapter and then added to the forecasts produced by the base model. This adapter must balance two requirements: sufficient expressive capacity and a compact parameter size to avoid overfitting. This design principle aligns with parameter-efficient fine-tuning methods, where a pre-trained base model is frozen and only a small adapter is updated.

Inspired by the Low-Rank Adaptation (LoRA) technique \cite{hu2022lora}, we implement the adapter as two low-rank projection matrices. Let the matrix of past prediction errors be denoted as $\epsilon\in \mathbb{R}^{d\times H}$ and the base model's forecast as $\hat{y}\in \mathbb{R}^{d \times H}$, where $d$ is the dimension of time series and $H$ is the prediction horizon.
This error $\epsilon$ is first projected into a lower-dimensional space and then mapped back to the original output dimension. The complete error module is formulated as:
\begin{equation}
    \hat{y}'= (ReLu(\epsilon\times W_1))W_2+\hat{y}
\end{equation}
where $\hat{y}'$ is the final, adjusted prediction and $W_1,W_2$ are two small matrices of shape $\mathbb{R}^{H\times r}$ and $R^{r\times H}$, respectively. And $r$ is much smaller than $H$. The use of two small matrices, $W_1$ and $W_2$, instead of some large weight matrices, drastically reduces the number of parameters.

Although more complex modules, such as gated networks or deeper architectures, could be considered, experiments discussed in Appendix \ref{section:err_model} indicate that these alternatives generally perform worse, with models containing more parameters often exhibiting poorer efficiency.
\begin{figure*}[!h]
    \centering
    \includegraphics[width=0.95\linewidth]{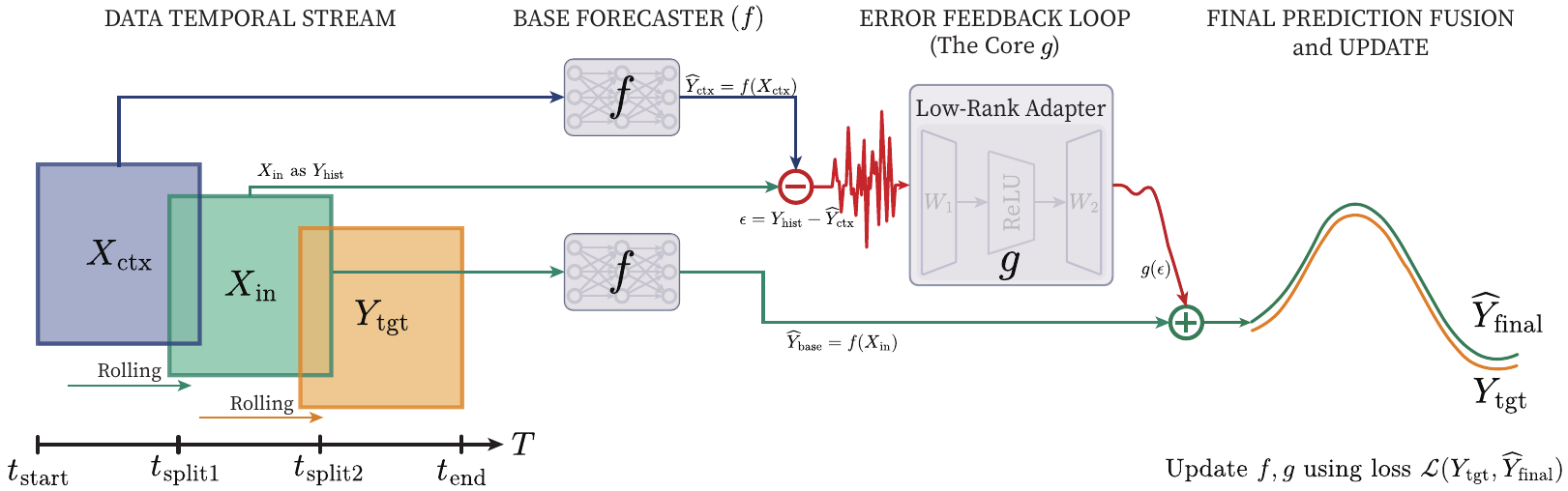}
    \caption{Illustration of TEFL joint training pipeline.}
    \label{fig:tefl}
\end{figure*}
\subsection{Training process}

Given that our framework introduces a learnable error correction module  $ g $  on top of the base forecaster  $ f $ , the choice of training strategy critically affects both performance and generalization. A naive approach would be to first train the base forecaster $f$ to convergence and then train the error module $g$ separately using its residuals. However, this two-stage pipeline risks a distributional mismatch: the error patterns observed during training may differ significantly from those encountered during rolling evaluation, especially if $f$ overfits to the shuffled training segments.

Conversely, jointly training $f$ and $g$ from scratch introduces another challenge: in early epochs, the predictions of $f$ are highly inaccurate, causing the inputs to $g$ (i.e., simulated residuals) to be dominated by noise. This hinders the error module’s ability to learn meaningful correction strategies.

To balance these concerns, we adopt a two-phase training protocol that ensures both signal quality and parameter co-adaptation.

\subsubsection{Phase 1: Warm-up with Temporally Structured Residuals}
We first warm up the base model $f$ using a composite loss that encourages its prediction residuals to exhibit temporal structure rather than white noise. The rationale is simple: if residuals are purely random, they carry no predictable signal for the error module; but if they retain systematic patterns (e.g., due to distribution shift or model bias), the error module can learn to correct them.

To quantify and promote such structure, we employ \textit{spectral flatness} (SF)—a differentiable measure of how "noise-like" a time series is \cite{dubnov2004generalization,gray2003spectral}. For a residual sequence $\epsilon = [\epsilon_0, \dots, \epsilon_{b-1}]$ over $b$ consecutive time steps, SF is computed as follows. First, apply the Discrete Fourier Transform (DFT):
\begin{equation}
    E_k = \sum_{n=0}^{b-1} \epsilon_n e^{-2\pi i k n / b}, \quad k = 0, \dots, b-1.
\end{equation}
Then compute the power spectrum $P_k = |E_k|^2$, and define spectral flatness as the ratio of its geometric mean to arithmetic mean:
\begin{equation}
    \mathrm{SF}(\epsilon) = \frac{\left( \prod_{k=0}^{b-1} P_k \right)^{1/b}}{\frac{1}{b} \sum_{k=0}^{b-1} P_k}.
\end{equation}
When the power spectrum is uniform (as in white noise), SF approaches 1; when energy is concentrated in a few frequencies (indicating structure or periodicity), SF approaches 0. By minimizing SF during warm-up, we encourage residuals to deviate from white noise and retain  temporal dependencies. Crucially, this requires training on temporally contiguous segments without shuffling, as spectral structure is only meaningful in the correct time order. Accordingly, our Phase 1 dataloader preserves the original sequence ordering.

The total warm-up loss for a batch of consecutive samples is:
\begin{equation}
    \mathcal{L}_{\text{warm}} = \ell(y, \hat{y}) + \alpha \cdot \mathrm{SF}(y - \hat{y}),
\end{equation}
where $\ell$ is the standard forecasting loss (e.g., MAE), and $\alpha > 0$ controls the regularization strength.

\subsubsection{Phase 2: Joint Training with Causal Error Simulation}
In the second phase, we jointly optimize the base model  $ f $  and the error module  $ g $ . For clarity, we assume a standard forecasting setting where both the input (lookback) window and the output (prediction) horizon have the same length  $ H $ —a common configuration in many time series benchmarks (e.g., predicting the next 96 steps given the past 96 steps).

Under this setting, to simulate the causal structure of rolling forecasting within a batch-compatible framework, each training sample is constructed from a contiguous segment of length  $ 3H $ , covering time steps  $ t-2H $  through  $ t+H-1 $ . This segment is partitioned into three consecutive blocks:
\begin{itemize}
    \item \textbf{Context window} ( $ X_{\text{ctx}} $ ):  $ [y_{t-2H}, \dots, y_{t-H-1}] $  — used as input to generate a forecast for the subsequent  $ H $  steps;
    \item \textbf{Input window} ( $ X_{\text{in}} $ ):  $ [y_{t-H}, \dots, y_{t-1}] $  — serves as the input for the current prediction;
    \item \textbf{Target} ( $ Y_{\text{tgt}} $ ):  $ [y_{t}, \dots, y_{t+H-1}] $  — the ground-truth values to be predicted.
\end{itemize}

The training procedure proceeds as follows:
\begin{enumerate}
    \item Apply the base model  $ f $  to  $ X_{\text{ctx}} $  to obtain a prediction  $ \hat{y}_{t-H:t-1} $  for time steps  $ t-H $  through  $ t-1 $ .
    \item Compute the residual vector  $ \epsilon = y_{t-H:t-1} - \hat{y}_{t-H:t-1} $ , where  $ y_{t-H:t-1} = [y_{t-H}, \dots, y_{t-1}] $  are the true observations. This error corresponds to a full  $ H $ -step forecast at time  $ t-H $ , and would be fully observable by time  $ t $  in a real rolling deployment.
    \item Apply  $ f $  to  $ X_{\text{in}} $  to produce the base forecast  $ \hat{y}_{t:t+H-1} $  for the target horizon.
    \item Pass  $ \epsilon $  through the error module  $ g $  and add its output as a correction:
    \begin{equation}
        \hat{y}'_{t:t+H-1} = f(X_{\text{in}}) + g(\epsilon).
    \end{equation}
    \item Compute the loss between  $ \hat{y}'_{t:t+H-1} $  and  $ Y_{\text{tgt}} = y_{t:t+H-1} $ , and backpropagate gradients through both  $ f $  and  $ g $ .
\end{enumerate}

This construction ensures two critical properties:
\begin{enumerate}
    \item \textbf{Causal validity}: The error signal  $ \epsilon $  is derived from a forecast whose entire ground truth is available at the time of the current prediction—respecting the temporal constraints of online evaluation.
    \item \textbf{Parameter consistency}: Since  $ \epsilon $  is computed using the current version of  $ f $, there is no mismatch between the error and the model state.
\end{enumerate}

Although our current implementation assumes equal input and output lengths ( $ H $ ), the core idea—using a fully observed error block for feedback—can be generalized to asymmetric settings with minor adjustments. The complete workflow is illustrated in Figure \ref{fig:tefl}. Ablation studies on the role of spectral flatness regularization, as well as comparisons with alternative training strategies are provided in Appendix \ref{section:method}.

\section{Theoretical Justification}
While our practical implementation employs highly nonlinear neural forecasters, a rigorous theoretical analysis of such models remains intractable. Nevertheless, the core mechanism behind TEFL can be rigorously demonstrated in an idealized setting: even an \textit{oracle} forecaster—one that computes the true conditional expectation—produces residuals with exploitable temporal structure, provided that (i) predictions are made using a finite input window (ii) the underlying system exhibits nonlinear state transitions, and (iii) observations are corrupted by additive noise. Under these conditions, observation noise obscures the true latent state, causing the optimal predictor to make systematic errors whose magnitude and sign depend on past noise realizations. As a result, the residual sequence  exhibits temporal correlations—not due to model misspecification, but as an inherent consequence of partial observability in nonlinear systems. Crucially, this structure can be leveraged for correction. As we show below, even a simple linear adapter can exploit these correlations to improve forecasting accuracy.

We analyze TEFL under a canonical nonlinear Gaussian state-space model:
\begin{align}
    x_t &= f(x_{t-1}) + \eta_{t-1}, \quad \eta_{t-1} \sim \mathcal{N}(0, \sigma_\eta^2), \\
    y_t &= x_t + \varepsilon_t, \quad \varepsilon_t \sim \mathcal{N}(0, \sigma_\varepsilon^2),
\end{align}
where  $ f \in C^2(\mathbb{R}) $ , the latent process  $ \{x_t\} $  is stationary with invariant distribution  $ \pi $ , and all noise variables are mutually independent. The forecaster observes only the noisy sequence  $ \{y_t\} $  and uses the MSE-optimal one-step predictor  $ g^*(y_{t-1}) = \mathbb{E}[y_t \mid y_{t-1}] $ .

Despite its optimality, the residual  $ r_t := y_t - g^*(y_{t-1}) $  exhibits temporal dependence due to observation noise. The following proposition quantifies this effect.

\begin{proposition}[Lag-1 Autocovariance of Residuals]
Consider the state-space model with observation noise ( $ \sigma_\varepsilon > 0 $ ) and nonlinear dynamics such that  $ \mu_f' := \mathbb{E}_{x \sim \pi}[f'(x)] \neq 0 $ . Then the prediction residuals of the optimal one-step forecaster satisfy

$$
\operatorname{Cov}(r_t, r_{t-1}) \neq 0.
$$

Thus, even an oracle predictor leaves structured, predictable residuals whenever observations are noisy and the dynamics are nonlinear—a common condition in real-world time series.
\end{proposition}

This non-zero autocorrelation implies that using past residuals to predict future residuals can improve the predictions of $g^*$. We formalize this for a linear TEFL adapter:  $ \hat{y}_t = g^*(y_{t-1}) + \beta r_{t-1} $ , where  $ \beta $  is estimated from  $ T $  samples via ordinary least squares.

\begin{theorem}[Finite-Sample MSE Reduction]
Let $\gamma = \operatorname{Cov}(r_t, r_{t-1})$, $V = \operatorname{Var}(r_t)$, and assume $|\gamma| > 0$. Denote the empirical TEFL coefficient by $\hat{\beta}_T = \frac{\sum_{t=2}^{T} r_t r_{t-1}}{\sum_{t=2}^{T} r_{t-1}^2}$. Then, under some mild conditions, there exist constant $c > 0$ such that for any $\delta \in (0,1)$, with probability at least $1 - \delta$,
\begin{multline}
\frac{1}{T-1} \sum_{t=2}^{T} \big( y_t - g^*(y_{t-1}) - \hat{\beta}_T r_{t-1} \big)^2 \\
\leq
\frac{1}{T-1} \sum_{t=2}^{T} r_t^2
- \frac{\gamma^2}{V}
+
c\left( \sqrt{\frac{\log(1/\delta)}{T}}
+
\frac{(\log(1/\delta))^2}{T}\right)
\end{multline}
Hence, for sufficiently large $T$, TEFL strictly improves over the base predictor with high probability.
\end{theorem}

These results establish that residual feedback is theoretically justified even under \emph{idealized conditions}: stationary dynamics, no distribution shift, and an oracle base predictor. Moreover, this performance gain can be captured using only a \emph{linear} adapter. In practical settings—where models are misspecified, training data suffer from distribution shifts, or the base forecaster is suboptimal—the residual signal is likely to contain \emph{even richer} information about systematic errors. Consequently, more expressive adapters (e.g., neural networks) might yield \emph{greater} improvements in such realistic, non-stationary regimes. Our analysis focuses on one-step history for clarity, but the presence of residual autocorrelation persists under multi-step conditioning in nonlinear or partially observed systems—a regime that covers most real-world forecasting tasks. Full proofs and additional discussion are provided in Appendix \ref{section:theo}. 

\section{Experiment}

\begin{figure*}[!h]
    \centering
    \includegraphics[width=0.95\linewidth]{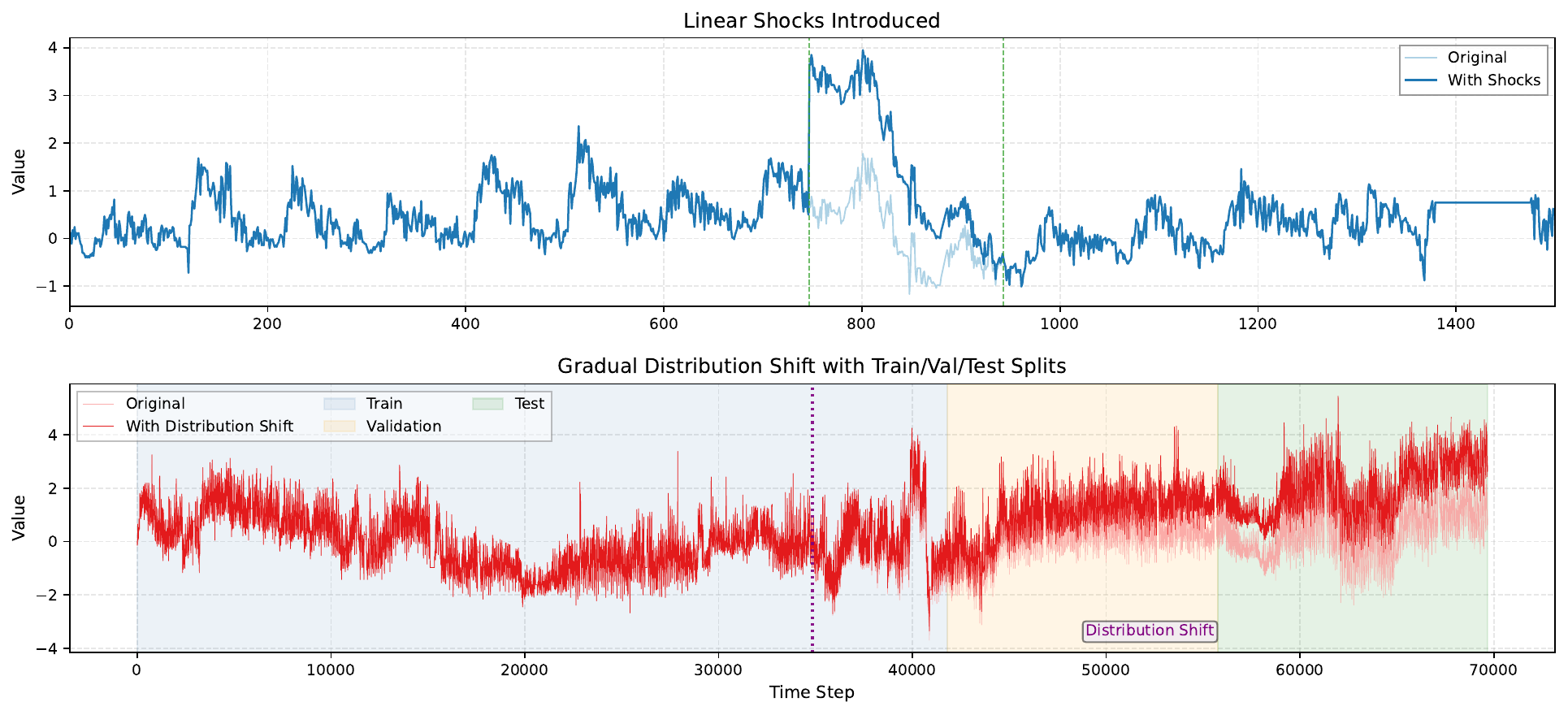}
    \caption{Examples of shock and distribution shift}
    \label{fig:shift}
\end{figure*}
\definecolor{lightblue}{RGB}{235, 245, 255}

\begin{table*}[!h]
\caption{Main results. Lower is better (↓). AVG denotes the average performance across five base models; IMP shows the relative improvement (\%) of +TEFL over Baseline.}
\fontsize{8}{8}\selectfont
\renewcommand{\arraystretch}{1.1}
\setlength{\tabcolsep}{2.2pt}
\centering
\begin{tabular}{@{}cccccccccccccccc@{}}
\toprule[1.5pt]
\multicolumn{2}{c}{Model} &
\multicolumn{2}{c}{SOFTS} &
\multicolumn{2}{c}{iTrans} &
\multicolumn{2}{c}{Dlinear} &
\multicolumn{2}{c}{TimeFilter} &
\multicolumn{2}{c}{Amplifier} &
\multicolumn{2}{c}{\shortstack{AVG  \scriptsize $\downarrow$}} &
\multicolumn{2}{c}{IMP} \\
\cmidrule[0.3pt]{3-16}
\multicolumn{2}{c}{Metric} &
MSE & MAE & MSE & MAE & MSE & MAE & MSE & MAE & MSE & MAE & MSE & MAE & MSE & MAE \\
\bottomrule[1pt]

                              & Baseline & 0.394 & 0.404 & 0.408 & 0.410 & 0.404 & 0.407 & 0.381 & 0.395 & 0.383 & 0.397 & \cellcolor{lightblue}0.394 & \cellcolor{lightblue}0.402 & \cellcolor{lightblue} & \cellcolor{lightblue} \\
\multirow{-2}{*}{ETTm1}       & +TEFL    & \textbf{0.362} & \textbf{0.378} & \textbf{0.380} & \textbf{0.388} & \textbf{0.372} & \textbf{0.384} & \textbf{0.366} & \textbf{0.376} & \textbf{0.366} & \textbf{0.378} & \cellcolor{lightblue}\textbf{0.369} & \cellcolor{lightblue}\textbf{0.381} & \cellcolor{lightblue}6.3\% & \cellcolor{lightblue}5.4\% \\ \midrule[0.3pt]
 
                              & Baseline & 0.303 & 0.330 & 0.289 & 0.334 & 0.346 & 0.396 & 0.273 & 0.323 & 0.279 & 0.325 & \cellcolor{lightblue}0.298 & \cellcolor{lightblue}0.341 & \cellcolor{lightblue} & \cellcolor{lightblue} \\
\multirow{-2}{*}{ETTm2}       & +TEFL    & \textbf{0.271} & \textbf{0.320} & \textbf{0.274} & \textbf{0.323} & \textbf{0.278} & \textbf{0.328} & \textbf{0.265} & \textbf{0.316} & \textbf{0.268} & \textbf{0.317} & \cellcolor{lightblue}\textbf{0.271} & \cellcolor{lightblue}\textbf{0.321} & \cellcolor{lightblue}9.1\% & \cellcolor{lightblue}6.0\% \\\midrule[0.3pt]

                              & Baseline & 0.256 & 0.278 & 0.258 & 0.279 & 0.266 & 0.316 & 0.241 & 0.271 & 0.246 & 0.273 & \cellcolor{lightblue}0.253 & \cellcolor{lightblue}0.283 & \cellcolor{lightblue} & \cellcolor{lightblue} \\
\multirow{-2}{*}{Weather}     & +TEFL    & \textbf{0.241} & \textbf{0.268} & \textbf{0.239} & \textbf{0.266} & \textbf{0.253} & \textbf{0.292} & \textbf{0.227} & \textbf{0.259} & \textbf{0.231} & \textbf{0.264} & \cellcolor{lightblue}\textbf{0.238} & \cellcolor{lightblue}\textbf{0.270} & \cellcolor{lightblue}6.0\% & \cellcolor{lightblue}4.7\% \\\midrule[0.3pt]

                              & Baseline & 0.174 & 0.265 & 0.179 & 0.272 & 0.226 & 0.320 & 0.160 & 0.259 & 0.173 & 0.267 & \cellcolor{lightblue}0.182 & \cellcolor{lightblue}0.277 & \cellcolor{lightblue} & \cellcolor{lightblue} \\
\multirow{-2}{*}{Electricity} & +TEFL    & \textbf{0.171} & \textbf{0.261} & \textbf{0.168} & \textbf{0.258} & \textbf{0.202} & \textbf{0.290} & \textbf{0.160} & \textbf{0.256} & \textbf{0.170} & \textbf{0.252} & \cellcolor{lightblue}\textbf{0.174} & \cellcolor{lightblue}\textbf{0.263} & \cellcolor{lightblue}4.5\% & \cellcolor{lightblue}4.7\% \\\midrule[0.3pt]

                              & Baseline & 0.154 & 0.242 & 0.157 & 0.243 & 0.197 & 0.284 & 0.153 & 0.237 & 0.171 & 0.253 & \cellcolor{lightblue}0.166 & \cellcolor{lightblue}0.252 & \cellcolor{lightblue} & \cellcolor{lightblue} \\
\multirow{-2}{*}{EURUSD}      & +TEFL    & \textbf{0.153} & \textbf{0.239} & \textbf{0.151} & \textbf{0.238} & \textbf{0.192} & \textbf{0.277} & \textbf{0.152} & \textbf{0.237} & \textbf{0.152} & \textbf{0.236} & \cellcolor{lightblue}\textbf{0.161} & \cellcolor{lightblue}\textbf{0.246} & \cellcolor{lightblue}3.4\% & \cellcolor{lightblue}2.3\% \\\midrule[0.3pt]

                              & Baseline & 0.883 & 0.517 & 0.897 & 0.522 & 0.848 & 0.541 & 0.878 & 0.515 & 0.857 & 0.514 & \cellcolor{lightblue}0.872 & \cellcolor{lightblue}0.522 & \cellcolor{lightblue} & \cellcolor{lightblue} \\
\multirow{-2}{*}{AQWan}       & +TEFL    & \textbf{0.825} & \textbf{0.501} & \textbf{0.839} & \textbf{0.503} & \textbf{0.822} & \textbf{0.506} & \textbf{0.831} & \textbf{0.501} & \textbf{0.815} & \textbf{0.491} & \cellcolor{lightblue}\textbf{0.826} & \cellcolor{lightblue}\textbf{0.500} & \cellcolor{lightblue}5.3\% & \cellcolor{lightblue}4.1\% \\\midrule[0.3pt]

                              & Baseline & 0.584 & 0.472 & 0.577 & 0.470 & 0.520 & 0.472 & 0.592 & 0.478 & 0.719 & 0.573 & \cellcolor{lightblue}0.599 & \cellcolor{lightblue}0.493 & \cellcolor{lightblue} & \cellcolor{lightblue} \\
\multirow{-2}{*}{ZafNoo}      & +TEFL    & \textbf{0.551} & \textbf{0.454} & \textbf{0.556} & \textbf{0.457} & \textbf{0.513} & \textbf{0.450} & \textbf{0.564} & \textbf{0.466} & \textbf{0.544} & \textbf{0.448} & \cellcolor{lightblue}\textbf{0.545} & \cellcolor{lightblue}\textbf{0.455} & \cellcolor{lightblue}8.9\% & \cellcolor{lightblue}7.8\% \\\midrule[0.3pt]

                              & Baseline & 0.784 & 0.527 & 0.795 & 0.532 & 0.748 & 0.552 & 0.773 & 0.525 & 0.756 & 0.523 & \cellcolor{lightblue}0.771 & \cellcolor{lightblue}0.532 & \cellcolor{lightblue} & \cellcolor{lightblue} \\
\multirow{-2}{*}{AQshunyi}    & +TEFL    & \textbf{0.721} & \textbf{0.500} & \textbf{0.745} & \textbf{0.512} & \textbf{0.721} & \textbf{0.515} & \textbf{0.732} & \textbf{0.509} & \textbf{0.719} & \textbf{0.499} & \cellcolor{lightblue}\textbf{0.728} & \cellcolor{lightblue}\textbf{0.507} & \cellcolor{lightblue}5.6\% & \cellcolor{lightblue}4.7\% \\\midrule[0.3pt]

                              & Baseline & 0.231 & 0.258 & 0.236 & 0.266 & 0.327 & 0.398 & 0.223 & 0.251 & 0.253 & 0.285 & \cellcolor{lightblue}0.254 & \cellcolor{lightblue}0.292 & \cellcolor{lightblue} & \cellcolor{lightblue} \\
\multirow{-2}{*}{Solar}       & +TEFL    & \textbf{0.205} & \textbf{0.257} & \textbf{0.218} & \textbf{0.246} & \textbf{0.253} & \textbf{0.245} & \textbf{0.221} & \textbf{0.251} & \textbf{0.231} & \textbf{0.262} & \cellcolor{lightblue}\textbf{0.224} & \cellcolor{lightblue}\textbf{0.252} & \cellcolor{lightblue}11.9\% & \cellcolor{lightblue}13.6\% \\\midrule[0.3pt]

                              & Baseline & 0.260 & 0.289 & 0.263 & 0.298 & 0.554 & 0.507 & 0.268 & 0.300 & 0.258 & 0.291 & \cellcolor{lightblue}0.321 & \cellcolor{lightblue}0.337 & \cellcolor{lightblue} & \cellcolor{lightblue} \\
\multirow{-2}{*}{CzeLan}      & +TEFL    & \textbf{0.250} & \textbf{0.278} & \textbf{0.255} & \textbf{0.281} & \textbf{0.305} & \textbf{0.339} & \textbf{0.258} & \textbf{0.288} & \textbf{0.245} & \textbf{0.279} & \cellcolor{lightblue}\textbf{0.263} & \cellcolor{lightblue}\textbf{0.293} & \cellcolor{lightblue}18.1\% & \cellcolor{lightblue}13.1\% \\
\bottomrule[1.5pt]
\end{tabular}
\label{tab:result1}
\end{table*}

\subsection{Set up}

To verify the effectiveness of the proposed TEFL method, we conducted experiments on ten datasets. These include commonly used time series forecasting datasets such as the ETT series, Weather, Electricity, etc. We also collected one financial dataset: the EURUSD dataset, which contains 15-minute interval data on the Euro/USD exchange rate and trading volume for 2022. Detailed dataset descriptions are provided in Appendix \ref{section:dataset}.

Our method is independent of the specific form of the base forecasting model. Therefore, to demonstrate its broad applicability, we selected five deep learning-based time series models: two classical models, DLinear  \cite{zeng2023transformers} and iTransformer \cite{liu2023itransformer}, three recently proposed models, SOFTS \cite{han2024softs}, TimesFilter \cite{hu2025timefilter}, and Amplifier \cite{fei2025amplifier}. All experiments are based on the open-source library TSLib \cite{wang2024deep}.

For the forecasting task, we follow previous work. The input window length is set to 96. We evaluate the models on four different output horizons: 96, 192, 336, and 720 steps. All experiments are repeated three times.
\subsection{Results}
Table \ref{tab:result1} reports the results across all datasets and models. For each dataset, the upper row shows the results of the original training method, and the lower row shows the results of TEFL method. These results represent the average performance across four forecasting horizons. The final four columns of the table present the average MSE, average MAE, and the average percentage improvement for all forecasting models.

The experimental results indicate that, in terms of average MSE and MAE, the TEFL method outperforms the corresponding baseline methods on all ten datasets. The extent of performance improvement varies across datasets. The most significant improvement is observed on the CzeLan dataset, where MSE and MAE are substantially reduced by 18.1\% and 13.1\%, respectively. Notable improvements are also achieved on the Solar dataset (MSE: 11.9\%, MAE: 13.6\%). Stable improvements ranging from 5\% to 9\% are observed on datasets including ETTm2, AQWan, Aqsunyi, ETTm1, and Weather. 

From the model perspective, the TEFL method consistently improves the performance of all five base forecasting models. This demonstrates that the optimization mechanism provided by TEFL offers good compatibility and enhancement for different types of model architectures.

\begin{table}[!h]
\caption{Results for shock data. Lower is better.}
\label{tab:shock}
\fontsize{8}{8}\selectfont  
\renewcommand{\arraystretch}{.8}
\setlength{\tabcolsep}{2.1pt}  
\centering
\begin{tabular}{cccccccc}
\toprule[1.5pt]
\multicolumn{2}{c}{Dataset} & \multicolumn{2}{c}{ETTm1} & \multicolumn{2}{c}{ETTm2} & \multicolumn{2}{c}{Weather} \\
\cmidrule[0.3pt]{3-8}
\multicolumn{2}{c}{Metric}  & MSE & MAE & MSE & MAE & MSE & MAE \\
\midrule[1pt]

                             & Baseline & 0.650 & 0.516 & 0.563 & 0.451 & 0.621 & 0.449 \\
\multirow{-2}{*}{Amplifier}  & +TEFL    & \textbf{0.579} & \textbf{0.474} & \textbf{0.505} & \textbf{0.425} & \textbf{0.539} & \textbf{0.404} \\
\cmidrule[0.3pt]{1-8}

                             & Baseline & 0.704 & 0.539 & 0.511 & 0.435 & 0.549 & 0.411 \\
\multirow{-2}{*}{iTrans}     & +TEFL    & \textbf{0.634} & \textbf{0.498} & \textbf{0.480} & \textbf{0.413} & \textbf{0.522} & \textbf{0.388} \\
\cmidrule[0.3pt]{1-8}

                             & Baseline & 0.711 & 0.560 & 0.733 & 0.598 & 0.619 & 0.527 \\
\multirow{-2}{*}{DLinear}    & +TEFL    & \textbf{0.628} & \textbf{0.500} & \textbf{0.511} & \textbf{0.450} & \textbf{0.549} & \textbf{0.429} \\
\cmidrule[0.3pt]{1-8}

                             & Baseline & 0.676 & 0.526 & 0.507 & 0.429 & 0.544 & 0.409 \\
\multirow{-2}{*}{SOFTS}      & +TEFL    & \textbf{0.618} & \textbf{0.493} & \textbf{0.471} & \textbf{0.402} & \textbf{0.501} & \textbf{0.372} \\
\cmidrule[0.3pt]{1-8}

                             & Baseline & 0.625 & 0.503 & 0.508 & 0.429 & 0.522 & 0.400 \\
\multirow{-2}{*}{TimeFilter} & +TEFL    & \textbf{0.577} & \textbf{0.464} & \textbf{0.477} & \textbf{0.407} & \textbf{0.495} & \textbf{0.371} \\
\cmidrule[0.3pt]{1-8}

                             & Baseline & 0.673 & 0.529 & 0.564 & 0.468 & 0.571 & 0.440 \\
\multirow{-2}{*}{AVG}        & +TEFL    & \textbf{0.607} & \textbf{0.486} & \textbf{0.489} & \textbf{0.419} & \textbf{0.521} & \textbf{0.393} \\
\cmidrule[0.3pt]{1-8}

\multicolumn{2}{c}{IMP (\%)} & 9.8\% & 8.1\% & 13.4\% & 10.5\% & 8.7\% & 10.6\% \\
\bottomrule[1.5pt]
\end{tabular}
\end{table}

\begin{table}[!h]
\caption{Results for data with distribution shift. Lower is better.}
\label{tab:shift}
\fontsize{8}{8}\selectfont
\renewcommand{\arraystretch}{.8}
\setlength{\tabcolsep}{2pt}
\centering
\begin{tabular}{cccccccc}
\toprule[1.5pt]
\multicolumn{2}{c}{Dataset} & \multicolumn{2}{c}{ETTm1} & \multicolumn{2}{c}{ETTm2} & \multicolumn{2}{c}{Weather} \\
\cmidrule[0.3pt]{3-8}
\multicolumn{2}{c}{Metric}  & MSE & MAE & MSE & MAE & MSE & MAE \\
\midrule[1pt]

                             & Baseline & 0.472 & 0.459 & 0.342 & 0.392 & 0.441 & 0.451 \\
\multirow{-2}{*}{Amplifier}  & +TEFL    & \textbf{0.435} & \textbf{0.429} & \textbf{0.283} & \textbf{0.340} & \textbf{0.286} & \textbf{0.333} \\
\cmidrule[0.3pt]{1-8}

                             & Baseline & 0.531 & 0.508 & 0.310 & 0.384 & 0.450 & 0.443 \\
\multirow{-2}{*}{iTrans}     & +TEFL    & \textbf{0.492} & \textbf{0.485} & \textbf{0.273} & \textbf{0.344} & \textbf{0.367} & \textbf{0.395} \\
\cmidrule[0.3pt]{1-8}

                             & Baseline & 0.589 & 0.565 & 0.358 & 0.428 & 0.562 & 0.571 \\
\multirow{-2}{*}{DLinear}    & +TEFL    & \textbf{0.433} & \textbf{0.438} & \textbf{0.277} & \textbf{0.337} & \textbf{0.308} & \textbf{0.351} \\
\cmidrule[0.3pt]{1-8}

                             & Baseline & 0.584 & 0.539 & 0.417 & 0.434 & 0.449 & 0.446 \\
\multirow{-2}{*}{SOFTS}      & +TEFL    & \textbf{0.483} & \textbf{0.473} & \textbf{0.291} & \textbf{0.364} & \textbf{0.408} & \textbf{0.423} \\
\cmidrule[0.3pt]{1-8}

                             & Baseline & 0.662 & 0.603 & 0.373 & 0.412 & 0.389 & 0.427 \\
\multirow{-2}{*}{TimeFilter} & +TEFL    & \textbf{0.534} & \textbf{0.503} & \textbf{0.277} & \textbf{0.344} & \textbf{0.332} & \textbf{0.382} \\
\cmidrule[0.3pt]{1-8}

                             & Baseline & 0.568 & 0.535 & 0.360 & 0.410 & 0.458 & 0.468 \\
\multirow{-2}{*}{AVG}        & +TEFL    & \textbf{0.476} & \textbf{0.466} & \textbf{0.280} & \textbf{0.346} & \textbf{0.340} & \textbf{0.377} \\
\cmidrule[0.3pt]{1-8}

\multicolumn{2}{c}{IMP (\%)} & 16.2\% & 13.0\% & 22.2\% & 15.7\% & 25.7\% & 19.5\% \\
\bottomrule[1.5pt]
\end{tabular}
\end{table}
\subsection{Robust to Shocks}
Shocks refer to sudden peaks in the time series. For example, in a series representing gold prices, an abrupt political crisis may cause a sharp immediate rise, followed by a gradual stabilization or decline as the situation calms \cite{chen1993joint}. Such sudden change could cause a significant increase in the prediction errors of recent time steps. Therefore, models that incorporate past prediction errors can directly use the change as input for forecasting future values, thereby improving prediction accuracy during shock periods.

To verify this, we conducted additional experiments on the ETTm1, ETTm2 and Weather datasets. Specifically, we artificially added 30 shocks into each dataset. Each shock lasts for 192 time steps and decays linearly, as illustrated in the upper subfigure of Figure \ref{fig:shift}. We retrained and tested both the baseline method and the TEFL method on these modified datasets. The results are presented in Table \ref{tab:shock}. 
For every model on these shock datasets, TEFL yields significantly lower MSE and MAE. The "IMP" row at the bottom of Table \ref{tab:shock} highlights an average error reduction of 9.8\% for ETTm1, 13.4\% for ETTm2, and 8.7\% for the Weather dataset. Besides, compared to the results in Table \ref{tab:result1}, our method yields a more pronounced accuracy improvement on the dataset containing shocks, which demonstrates that incorporating past prediction errors into the current input can better handle shocks.

\subsection{Robust to distribution shift}
Typically, past prediction errors can partially reflect distribution shifts. Therefore, our proposed TEFL method, which incorporates previous errors into the input, should be capable of modeling distribution shifts to some extent. To verify the robustness of our method against distribution shifts, we conducted additional experiments on the ETTm1, ETTm2, and Weather datasets.

Specifically, we artificially introduced a distribution shift into these datasets. For the normalized dataset, we applied a linearly increasing distribution shift to the latter half of the data, as shown in the lower part of Figure \ref{fig:shift}. It can be observed that a progressively more pronounced distribution shift exists from the end of the training set through to the end of the test set. We repeated experiments on these datasets. And the results are presented in Table \ref{tab:shift}.

As shown in Table \ref{tab:shift}, our method achieves a substantial improvement over directly training these models. The reduction in MAE reaches 13.0\%, 15.7\%, and 19.5\% on the three datasets, respectively. Notably, this improvement is also larger than that observed in Table \ref{tab:result1}. This result underscores the strong adaptability of TEFL method to distribution shifts.


\section{Conclusion}
This paper presents TEFL, a novel training paradigm for time series forecasting that explicitly incorporates historically observed prediction errors as corrective signals during learning. By simulating the causal structure of rolling forecasting, TEFL bridges the gap between conventional shuffled-batch training and real-world deployment dynamics. Extensive experiments across diverse datasets and backbone models demonstrate that TEFL consistently improves forecasting accuracy, achieving up to 10\% error reduction in several settings. Moreover, TEFL exhibits enhanced robustness to distribution shifts and abrupt changes.

Our framework offers a practical and generalizable approach to adaptive forecasting, aligning offline training with online evaluation. Future directions include extending the error feedback mechanism to online learning, developing more effective training strategies, and applying this paradigm to other sequential prediction tasks such as spatiotemporal modeling.
\section*{Impact Statement}


This paper presents work whose goal is to advance the field of 
Machine Learning. There are many potential societal consequences 
of our work, none which we feel must be specifically highlighted here.


\nocite{langley00}

\bibliography{example_paper}
\bibliographystyle{icml2025}

\newpage
\appendix
\onecolumn
\section{Theoretical analysis}
\label{section:theo}
Our theoretical analysis is grounded in a canonical hidden-state time series model under observational noise. Under that setting, even for an oracle predictor— one that knows the true data-generating mechanism- the resulting prediction error on the observed sequence exhibits temporal correlation across consecutive time steps. This correlation arises from two sources: (i) the intrinsic dependence in the latent dynamics, and (ii) the persistent influence of past observation noise through the filtering history.

This implies that even an optimal predictor leaves structured, predictable residuals in the presence of observation noise. Consequently, a secondary correction mechanism that leverages recent residual history can consistently reduce prediction error, without requiring any modification to the base forecaster. This insight forms the core theoretical justification for our approach: residual-guided refinement is provably beneficial even under oracle conditions, provided the observation model is noisy and the transition dynamic is nonlinear—a ubiquitous scenario in real-world time series.
In the following, we formalize this intuition in a nonlinear-Gaussian setting and establish finite-sample guarantees for residual-based correction.
\begin{proposition}[Lag-1 Autocovariance of Residuals]
Consider the scalar state-space model
\begin{align}
    x_t &= f(x_{t-1}) + \eta_{t-1}, \quad \eta_{t-1} \sim \mathcal{N}(0, \sigma_\eta^2), \\
    y_t &= x_t + \varepsilon_t, \quad \varepsilon_t \sim \mathcal{N}(0, \sigma_\varepsilon^2),
\end{align}
where  $  f \in C^2(\mathbb{R})  $  has bounded first and second derivatives,  $  \{x_t\}  $  is a geometrically ergodic Markov chain with invariant density  $  \pi(x) > 0  $  for all  $  x \in \mathbb{R}  $ , and all noise variables are mutually independent. Define the optimal one-step predictor

$$
g^*(y_{t-1}) := \mathbb{E}[y_t \mid y_{t-1}],
$$

and the prediction residual

$$
r_t := y_t - g^*(y_{t-1}).
$$

Assume that  $  \mu_f' := \mathbb{E}_{x \sim \pi}[f'(x)] \neq 0  $  and  $  \sigma_\varepsilon > 0  $ . Then, as  $  \sigma_\varepsilon \to 0  $ ,

$$
\operatorname{Cov}(r_t, r_{t-1}) = -\mu_f' \sigma_\varepsilon^2 + o(\sigma_\varepsilon^2).
$$

In particular,  $  \operatorname{Cov}(r_t, r_{t-1}) \neq 0  $  for all sufficiently small  $  \sigma_\varepsilon > 0  $ .
\end{proposition}

\begin{proof}
We proceed in three steps: (i) asymptotic expansion of the conditional expectation under small observation noise, (ii) decomposition of the residuals, and (iii) precise covariance calculation.

\medskip

\noindent\textbf{Step 1: Small-noise expansion of  $  \mathbb{E}[f(x) \mid y]  $ .}

Consider the static observation model

$$
x \sim \pi(x), \quad y = x + \varepsilon, \quad \varepsilon \sim \mathcal{N}(0, \sigma_\varepsilon^2),
$$

with  $  \pi \in C^2(\mathbb{R})  $ ,  $  \pi(x) > 0  $ , and  $  f \in C^2(\mathbb{R})  $  having bounded derivatives up to order 2. The posterior density is

$$
p(x \mid y) = \frac{\pi(x) \phi\big((y - x)/\sigma_\varepsilon\big)}{\int_{\mathbb{R}} \pi(u) \phi\big((y - u)/\sigma_\varepsilon\big)\, du},
\quad \text{where } \phi(z) = (2\pi)^{-1/2} e^{-z^2/2}.
$$

Using the change of variable  $  x = y - \sigma_\varepsilon z  $ , we obtain

$$
\mathbb{E}[f(x) \mid y] = 
\frac{\int_{\mathbb{R}} f(y - \sigma_\varepsilon z) \pi(y - \sigma_\varepsilon z) \phi(z)\, dz}
     {\int_{\mathbb{R}} \pi(y - \sigma_\varepsilon z) \phi(z)\, dz}.
$$

By Taylor's theorem with remainder, for some constants  $  C_f, C_\pi > 0  $ ,
\begin{align*}
f(y - \sigma_\varepsilon z) &= f(y) - \sigma_\varepsilon f'(y) z + \tfrac{1}{2} \sigma_\varepsilon^2 f''(y) z^2 + R_f(z), \\
\pi(y - \sigma_\varepsilon z) &= \pi(y) - \sigma_\varepsilon \pi'(y) z + \tfrac{1}{2} \sigma_\varepsilon^2 \pi''(y) z^2 + R_\pi(z),
\end{align*}
where  $  |R_f(z)| \leq C_f \sigma_\varepsilon^3 |z|^3  $  and  $  |R_\pi(z)| \leq C_\pi \sigma_\varepsilon^3 |z|^3  $ .

Since  $  \mathbb{E}_\phi[z] = 0  $ ,  $  \mathbb{E}_\phi[z^2] = 1  $ , and  $  \mathbb{E}_\phi[|z|^3] < \infty  $ , integration against  $  \phi(z)  $  yields
\begin{align*}
\text{Denominator} &= \pi(y) + \tfrac{1}{2} \sigma_\varepsilon^2 \pi''(y) + O(\sigma_\varepsilon^3), \\
\text{Numerator}   &= f(y)\pi(y) + \sigma_\varepsilon^2 \Big( \tfrac{1}{2} f(y)\pi''(y) + \tfrac{1}{2} f''(y)\pi(y) + f'(y)\pi'(y) \Big) + O(\sigma_\varepsilon^3).
\end{align*}

Dividing and using  $  (a + \delta)/(b + \epsilon) = a/b + (\delta b - a \epsilon)/b^2 + O(\delta^2 + \epsilon^2)  $ , we obtain the pointwise expansion
\begin{equation}\label{eq:cond_expansion}
\mathbb{E}[f(x) \mid y] = f(y) + \sigma_\varepsilon^2 \left( \tfrac{1}{2} f''(y) + f'(y) \frac{\pi'(y)}{\pi(y)} \right) + O(\sigma_\varepsilon^3),
\end{equation}
valid for all  $  y \in \mathbb{R}  $ , with the remainder uniform on compact sets. Since  $  \pi  $  is the stationary density of a geometrically ergodic chain and  $  f'  $  is bounded, all expectations below are well-defined.

\medskip

\noindent\textbf{Step 2: Residual decomposition.}

Recall that  $  y_t = f(x_{t-1}) + \eta_{t-1} + \varepsilon_t  $ , so

$$
r_t = y_t - \mathbb{E}[y_t \mid y_{t-1}] = f(x_{t-1}) - \mathbb{E}[f(x_{t-1}) \mid y_{t-1}] + \eta_{t-1} + \varepsilon_t.
$$

Since  $  y_{t-1} = x_{t-1} + \varepsilon_{t-1}  $ , we may write  $  x_{t-1} = y_{t-1} - \varepsilon_{t-1}  $ . Applying Taylor's theorem to  $  f  $  around  $  y_{t-1}  $  gives

$$
f(x_{t-1}) = f(y_{t-1}) - f'(y_{t-1}) \varepsilon_{t-1} + \tfrac{1}{2} f''(y_{t-1}) \varepsilon_{t-1}^2 + O_p(\sigma_\varepsilon^3).
$$

On the other hand, by the small-noise expansion \eqref{eq:cond_expansion},

$$
\mathbb{E}[f(x_{t-1}) \mid y_{t-1}] = f(y_{t-1}) + O(\sigma_\varepsilon^2).
$$

Subtracting these two expressions yields

$$
f(x_{t-1}) - \mathbb{E}[f(x_{t-1}) \mid y_{t-1}] = -f'(y_{t-1}) \varepsilon_{t-1} + O_p(\sigma_\varepsilon^2).
$$

Therefore, the residual admits the asymptotic representation
\begin{equation}\label{eq:rt_exp}
r_t = -f'(y_{t-1}) \varepsilon_{t-1} + \eta_{t-1} + \varepsilon_t + O_p(\sigma_\varepsilon^2).
\end{equation}
Analogously,
\begin{equation}\label{eq:rtm1_exp}
r_{t-1} = -f'(y_{t-2}) \varepsilon_{t-2} + \eta_{t-2} + \varepsilon_{t-1} + O_p(\sigma_\varepsilon^2).
\end{equation}

\medskip

\noindent\textbf{Step 3: Covariance calculation.}

Since  $ \mathbb{E}[r_t] = \mathbb{E}[r_{t-1}] = 0 $  by construction of the optimal predictor, we have

$$
\operatorname{Cov}(r_t, r_{t-1}) = \mathbb{E}[r_t r_{t-1}].
$$

Substituting the expansions \eqref{eq:rt_exp} and \eqref{eq:rtm1_exp}, and using the mutual independence of all noise variables across time (in particular,  $ \varepsilon_s $ ,  $ \eta_s $  are independent of  $ \{y_u : u \leq s-1\} $ ), we observe that all cross terms vanish in expectation except those involving the same noise instance.

The only non-vanishing contribution at order  $ \sigma_\varepsilon^2 $  arises from the product of  $ -f'(y_{t-1}) \varepsilon_{t-1} $  (from  $ r_t $ ) and  $ \varepsilon_{t-1} $  (from  $ r_{t-1} $ ). All other products involve independent zero-mean random variables or higher-order remainders. Hence,

$$
\mathbb{E}[r_t r_{t-1}] = \mathbb{E}\big[ -f'(y_{t-1}) \varepsilon_{t-1} \cdot \varepsilon_{t-1} \big] + o(\sigma_\varepsilon^2)
= -\mathbb{E}\big[ f'(y_{t-1}) \varepsilon_{t-1}^2 \big] + o(\sigma_\varepsilon^2).
$$




Since  $ y_{t-1} = x_{t-1} + \varepsilon_{t-1} $, expanding  $ f' $  around  $ x_{t-1} $ ,

$$
f'(x_{t-1} + \varepsilon_{t-1}) = f'(x_{t-1}) + f''(x_{t-1}) \varepsilon_{t-1} + O_p(\sigma_\varepsilon^2),
$$

and using  $ \mathbb{E}[\varepsilon_{t-1}] = 0 $ ,  $ \mathbb{E}[\varepsilon_{t-1}^2] = \sigma_\varepsilon^2 $ ,  $ \mathbb{E}[\varepsilon_{t-1}^3] = 0 $ , we obtain

$$
\mathbb{E}[f'(x_{t-1} + \varepsilon_{t-1}) \varepsilon_{t-1}^2]
= f'(x_{t-1}) \sigma_\varepsilon^2 + O(\sigma_\varepsilon^4).
$$

Taking expectation over the stationary distribution  $ \pi $  of  $ x_{t-1} $  yields

$$
\mathbb{E}[f'(y_{t-1}) \varepsilon_{t-1}^2] = \mathbb{E}_{x \sim \pi}[f'(x)] \, \sigma_\varepsilon^2 + o(\sigma_\varepsilon^2).
$$

Combining the above, we conclude

$$
\operatorname{Cov}(r_t, r_{t-1}) = -\mu_f' \, \sigma_\varepsilon^2 + o(\sigma_\varepsilon^2),
\quad \text{where } \mu_f' := \mathbb{E}_{x \sim \pi}[f'(x)].
$$

In particular, if  $ f $  is nonlinear and the stationary measure  $ \pi $  is non-degenerate, then  $ \mu_f' \neq 0 $  in general, implying  $ \operatorname{Cov}(r_t, r_{t-1}) \neq 0 $  for all sufficiently small  $ \sigma_\varepsilon > 0 $ .
\end{proof}









\subsection{Finite-Sample Guarantee for Linear TEFL}

\begin{theorem}[Finite-Sample MSE Reduction by TEFL]
\label{thm:tefl_finite_sample}
Consider the nonlinear Gaussian state-space model:

$$
x_t = f(x_{t-1}) + \eta_{t-1}, \quad y_t = x_t + \varepsilon_t,
$$

where  $  f \in C^2(\mathbb{R})  $ ,  $  \eta_t \sim \mathcal{N}(0,\sigma_\eta^2)  $ , and  $  \varepsilon_t \sim \mathcal{N}(0,\sigma_\varepsilon^2)  $  are mutually independent.  
Assume $  |f'(x)| \leq L < 1  $  for all  $  x \in \mathbb{R}  $, then, the joint process  $  \{(x_t, y_t)\}_{t \geq 0}  $  is stationary and geometrically ergodic. Additionally, we assume $r_t$ is subguassion and the stationaty distribution of $x_t$, denoted as $\pi(x)$ is also subguassion.
Define the base predictor  $  \hat{y}_t^{\mathrm{base}} = g^*(y_{t-1})  $  and its residual  $  r_t = y_t - g^*(y_{t-1})  $ .  
The TEFL predictor employs a linear adapter:

$$
\hat{y}_t^{\mathrm{TEFL}} = g^*(y_{t-1}) + \beta r_{t-1},
$$

with the empirical coefficient computed from  $  T  $  observations as

$$
\hat{\beta}_T = \frac{\sum_{t=2}^{T} r_t r_{t-1}}{\sum_{t=2}^{T} r_{t-1}^2}.
$$

Let the population quantities be

$$
\gamma := Cov(r_t, r_{t-1}), V := Var(r_t),  \rho_1 := \frac{\gamma}{V}.
$$

Assume  $  V > 0  $  and  $  |\rho_1| > 0  $ —conditions satisfied whenever  $  f  $  is nonlinear and  $  \sigma_\varepsilon > 0  $ .

Then there exist constants  $  c> 0  $ , depending only on the geometric ergodicity rate of  $  \{r_t\}  $  and the noise variances, such that for any  $  \delta \in (0,1)  $ , with probability at least  $  1 - \delta  $ ,
\begin{equation}
\label{eq:finite_sample_bound}
\frac{1}{T-1} \sum_{t=2}^{T} \bigl( y_t - g^*(y_{t-1}) - \hat{\beta}_T r_{t-1} \bigr)^2
\leq
\underbrace{\frac{1}{T-1} \sum_{t=2}^{T} r_t^2}_{\text{Base MSE}}
-
V \rho_1^2
+
c\left( \sqrt{\frac{\log(1/\delta)}{T}}
+
\frac{(\log(1/\delta))^2}{T}\right)
\end{equation}
In particular, for sufficiently large  $  T  $ , the right-hand side is strictly smaller than the base MSE with high probability.
\end{theorem}

\begin{proof}
\textbf{Step 1: Error decomposition.}  

Let  $  \beta^* = \gamma / V  $ . A standard algebraic identity yields

$$
\frac{1}{T-1} \sum_{t=2}^T (r_t - \hat{\beta}_T r_{t-1})^2
= \text{(A)} + (\hat{\beta}_T - \beta^*)^2 \text{(B)} + 2(\hat{\beta}_T - \beta^*) \text{(C)},
$$

where

$$
\text{(A)} = \frac{1}{T-1} \sum_{t=2}^T (r_t - \beta^* r_{t-1})^2, \quad
\text{(B)} = \frac{1}{T-1} \sum_{t=2}^T r_{t-1}^2, \quad
\text{(C)} = \frac{1}{T-1} \sum_{t=2}^T (r_t - \beta^* r_{t-1}) r_{t-1}.
$$

By construction of  $  \beta^*  $ , we have  $  \mathbb{E}[\text{(C)}] = 0  $  and  $  \mathbb{E}[\text{(A)}] = V - \gamma^2/V = V(1 - \rho_1^2)  $ .

\medskip

\textbf{Step 2: Mixing.}  

We define $X_t=(x_t,\epsilon_t)$, then $X_t$ is stationary and geometrically ergodic, as a result, $(X_t,X_{t-1})$ is also stationary and geometrically ergodic. According to Theorem 3.1 in \citet{2005Basic}, $(X_t,X_{t-1})$ is $\beta$ -mixing. $r_t$ is a measurable function of $(X_t,X_{t-1})$, as a result, is also $\beta$ -mixing.

Besides, the residual process  $r_t= y_t - g^*(y_{t-1}) =x_t+\epsilon_t-g^*(y_{t-1})$ is subguassion. Therefore:

$$
U_t = r_t r_{t-1} - \gamma, \quad
V_t = r_{t-1}^2 - V, \quad
W_t = (r_t - \beta^* r_{t-1}) r_{t-1}
$$

are stationary, exponentially  $ \beta $ -mixing, and sub-exponential.

\medskip

\textbf{Step 3: Concentration inequalities.}  

By Theorem 1 of \citet{merlevede2011bernstein}, for any  $  \delta \in (0,1)  $ , there exists a constant  $  c > 0  $ , depending only on the mixing rate and sub-exponential norms, such that with probability at least  $  1 - \delta  $ ,

$$
\max\left\{
\left| \frac{1}{T-1} \sum_{t=2}^T U_t \right|,
\left| \frac{1}{T-1} \sum_{t=2}^T V_t \right|,
\left| \frac{1}{T-1} \sum_{t=2}^T W_t \right|
\right\}
\leq
c \left( \sqrt{\frac{\log(1/\delta)}{T}} + \frac{(\log(1/\delta))^2}{T} \right).
$$

Denote this upper bound by

$$
\Delta_T := c \left( \sqrt{\frac{\log(1/\delta)}{T}} + \frac{(\log(1/\delta))^2}{T} \right).
$$

On this event, we have

$$
\left| \frac{1}{T-1} \sum_{t=2}^T r_t r_{t-1} - \gamma \right| \leq \Delta_T,
\qquad
\left| \frac{1}{T-1} \sum_{t=2}^T r_{t-1}^2 - V \right| \leq \Delta_T.
$$

Since  $  V > 0  $ , for  $  T  $  large enough such that  $  \Delta_T < V/2  $ , it follows that

$$
|\hat{\beta}_T - \beta^*|
= \left| \frac{\frac{1}{T-1} \sum r_t r_{t-1}}{\frac{1}{T-1} \sum r_{t-1}^2} - \frac{\gamma}{V} \right|
\leq \frac{V \Delta_T + |\gamma| \Delta_T}{V(V - \Delta_T)}
\leq \frac{2(V + |\gamma|)}{V^2} \Delta_T
=: C_0 \Delta_T.
$$

Moreover,

$$
|\text{(A)} - \mathbb{E}[\text{(A)}]| \leq \Delta_T, \quad
|\text{(B)} - V| \leq \Delta_T, \quad
|\text{(C)}| \leq \Delta_T.
$$

\medskip

\textbf{Step 4: Final bound.}  

Substituting into the error decomposition gives, with probability at least  $  1 - \delta  $ ,
\begin{align*}
\frac{1}{T-1} \sum_{t=2}^T (r_t - \hat{\beta}_T r_{t-1})^2
&\leq \mathbb{E}[\text{(A)}] + \Delta_T + (C_0 \Delta_T)^2 (V + \Delta_T) + 2 C_0 \Delta_T \cdot \Delta_T \\
&\leq V(1 - \rho_1^2) + \Delta_T + C_1 \Delta_T^2,
\end{align*}

Note that

$$
\frac{1}{T-1} \sum_{t=2}^{T} r_t^2 = V + \xi_T,
$$

with   $  |\xi_T| \leq \Delta_T  $   on the same high-probability event.

Then,

$$
\frac{1}{T-1} \sum_{t=2}^T (r_t - \hat{\beta}_T r_{t-1})^2 \leq \frac{1}{T-1} \sum_{t=2}^{T} r_t^2 - V\rho_1^2 + \Delta_T + C_1 \Delta_T^2
$$

The quadratic term $\Delta_T^2$ is dominated by $\Delta_T$ for all sufficiently large $T$, and can thus be absorbed into the leading-order error term. Hence, there exists an constant $c$ such that the finite-sample excess risk bound takes the simplified form shown in \eqref{eq:finite_sample_bound}.

\end{proof}

\medskip

This result shows that TEFL achieves a non-vanishing MSE reduction of  $ \gamma^2 / V = \rho_1^2 V $ in expectation, and this gain dominates the  $ \mathcal{O}(1/\sqrt{T}) $  estimation error with high probability for sufficiently large  $ T $ . Moreover, the magnitude of this in-sample improvement closely matches that of the generalization error  $ \mathbb{E}[(y_{T+1} - \hat{y}_{T+1}^{\mathrm{TEFL}})^2] $ , differing only by higher-order terms that vanish as  $ T \to \infty $ .

\subsection{Theorem 1 in \citet{merlevede2011bernstein}}

We first recall the truncation function used in the statement of the theorem.

For any  $ M > 0 $ , define the truncation function  $ \phi_M : \mathbb{R} \to \mathbb{R} $  by

$$
\phi_M(x) = (x \wedge M) \vee (-M).
$$

Let  $ (X_j)_{j \in \mathbb{Z}} $  be a sequence of centered real-valued random variables.
Define the quantity
\begin{equation}
\label{eq:def_V}
V := \sup_{M > 0} \sup_{i > 0} \left( Var(\phi_M(X_i)) + 2 \sum_{j > i} \bigl| Cov(\phi_M(X_i), \phi_M(X_j)) \bigr| \right).
\end{equation}

Assume the following conditions hold:

\begin{enumerate}
    \item There exist constants  $ a > 0 $ ,  $ c > 0 $ , and  $ \gamma_1 > 0 $  such that the  $\tau$ -mixing coefficients satisfy

$$
    \tau(x) \leq a \exp(-c x^{\gamma_1}) \quad \text{for all } x \geq 1,
$$

    \item There exist constants  $ b > 0 $  and  $ \gamma_2 \in (0, \infty] $  such that

$$
    \sup_{k > 0} \mathbb{P}(|X_k| > t) \leq \exp\bigl(1 - (t/b)^{\gamma_2}\bigr) \quad \text{for all } t > 0.
$$

    \item Define  $ \gamma > 0 $  by

$$
    \frac{1}{\gamma} = \frac{1}{\gamma_1} + \frac{1}{\gamma_2},
$$

    and $ \gamma < 1 $ .
\end{enumerate}

Then we have the following Bernstein-type inequality.

\begin{theorem}[Theorem 1 in \citet{merlevede2011bernstein}]
\label{thm:bernstein_merlevede}
Under above assumptions, the constant  $ V $  defined in \eqref{eq:def_V} is finite. Moreover, for any  $ n \geq 4 $ , there exist positive constants  $ C_1, C_2, C_3, C_4 $ , depending only on  $ a, b, c, \gamma_1, \gamma_2 $ , such that for all  $ x > 0 $ ,

$$
\mathbb{P}\left( \max_{1 \leq j \leq n} |S_j| \geq x \right)
\leq
n \exp\!\left( -\frac{x^\gamma}{C_1} \right)
+
\exp\!\left( -\frac{x^2}{C_2 (1 + n V)} \right)
+
\exp\!\left( -\frac{x^2}{C_3 n} \exp\!\left( \frac{x^{\gamma(1 - \gamma)}}{C_4 (\log x)^\gamma} \right) \right),
$$

where  $ S_j = X_1 + X_2 + \cdots + X_j $ .
\end{theorem}
\section{Distinction from Online Adaptation and Temporal Generalization Approaches}

While TEFL enhances forecasting through error feedback, it differs fundamentally from methods that adapt models during deployment or explicitly model temporal distribution shifts. We briefly contrast our approach with two related paradigms.

First, \textit{online adaptation} techniques update model parameters at inference time using newly observed data. A common design principle is to restrict updates to a small subset of parameters to maintain stability. For example, some works adjust only ensemble weights \cite{wen2023onenet}, adapter layers \cite{guo2024online,lee2025lightweight,huang2025online}, or student network components \cite{lau2025fast}. Full-model fine-tuning is generally avoided due to the risk of catastrophic forgetting or gradient instability when learning from single or few samples.

To mitigate noisy gradient estimates, several strategies have been proposed. Experience replay \cite{lau2025fast} incorporates historical observations into the loss computation, thereby reducing variance. Other approaches, such as FSNet \cite{pham2023learning}, maintain dual pathways: one for rapid adaptation using recent data, and another for slow, stable updates based on accumulated experience. Notably, certain methods bypass backpropagation entirely; ELF \cite{lee2025lightweight}, for instance, formulates the adapter as a linear regressor whose optimal solution can be computed analytically from a small batch of recent points.

Second, \textit{temporal domain generalization (TDG)} seeks to improve robustness under continuously evolving data distributions by modeling the dynamics of domain change. Representative approaches include: Gradient Interpolation (GI) \cite{nasery2021training}, which enforces temporal smoothness in the loss landscape; LSSAE \cite{qin2022generalizing}, which separates time-varying and invariant latent factors; and DRAIN \cite{bai2022temporal}, which uses a recurrent network to generate future model parameters. More recent work leverages dynamical systems theory—e.g., Koodos \cite{cai2024continuous} and FreKoo \cite{yu2025learning} employ the Koopman operator to embed nonlinear temporal evolution into a linear predictive space, while STAD \cite{schirmer2024temporal} treats classifier weights as drifting states governed by a state-space process.

In summary, both online adaptation and TDG require either parameter updates during inference or explicit modeling of how the data distribution evolves over time. In contrast, TEFL operates entirely offline: it treats past prediction residuals—naturally available in rolling evaluation—as auxiliary inputs to a fixed model. No parameter modification, experience buffer, or temporal dynamics model is needed, making TEFL lightweight, universally compatible, and free of test-time computational overhead.
\section{Datasets}
\label{section:dataset}
The datasets used in the experiments are as follows.

The ETTm1 and ETTm2 \cite{liu2023itransformer} datasets contain the temperature and electric load data recorded every 15 minutes from transformers in two regions of China, covering the period from 2016 to 2018.

The Weather \cite{liu2023itransformer} dataset records 21 different meteorological indicators every 10 minutes across Germany for the year 2020. Key indicators include air temperature and visibility.

The Electricity dataset \cite{liu2023itransformer} contains the hourly electricity consumption records (in kW·h) of 321 clients. It is sourced from the UCL Machine Learning Repository and covers the period from 2012 to 2014.

The AQShunyi and AQWan datasets \cite{zhang2017cautionary, qiu2024tfb} record the hourly air quality in Beijing from 2013 to 2017, including 11 factors such as PM2.5 and PM10.

The ZafNoo and CzeLan datasets \cite{poyatos2020global} record plant transpiration data every 30 minutes.

The \href{https://www.kaggle.com/datasets/godfrey03/eurusd-15-minute-ohlc-forex-data} {EURUSD} dataset records 6 indicators for the Euro/USD exchange rate, including the opening price, closing price, and trading volume, at 15-minute intervals from 2023 to 2024.

\section{Experiment details}

We provide the pseudo code for the training process of our method in Algorithm \ref{alg:tefl}.

\begin{algorithm}[t]
\caption{Training Procedure of TEFL}
\label{alg:tefl}
\begin{algorithmic}[1]
\REQUIRE Time series  $ \mathcal{D} = \{y_1, y_2, \dots, y_T\} $ , horizon  $ H $ , warm-up epochs  $ E_w $ , total epochs  $ E $ , regularization weight  $ \alpha $ 
\ENSURE Trained base model  $ f $  and error module  $ g $ 

\vspace{0.5em}
\textbf{// Phase 1: Warm-up base model}
\FOR{epoch = 1 \textbf{to}  $ E_w $ }
    \FOR{each batch of samples  $ \{(X_i, Y_i)\}_{i=1}^B $  from  $ \mathcal{D} $ }
        \STATE //  $ X_i = [y_{t-H}, \dots, y_{t-1}] $ ,  $ Y_i = [y_t, \dots, y_{t+H-1}] $ 
        \STATE  $ \hat{Y}_i \gets f(X_i) $ 
        \STATE  $ \epsilon_i \gets Y_i - \hat{Y}_i $ 
        \STATE  $ \mathcal{L}_{\text{acc}} \gets \frac{1}{B} \sum_{i=1}^B \| Y_i - \hat{Y}_i \|_1 $ 
        \STATE  $ \mathcal{L}_{\text{sf}} \gets \mathrm{SF}(\epsilon_1, \dots, \epsilon_B) $ 
        \STATE  $ \mathcal{L} \gets \mathcal{L}_{\text{acc}} + \alpha \cdot \mathcal{L}_{\text{sf}} $ 
        \STATE Update  $ f $  via  $ \nabla_f \mathcal{L} $ 
    \ENDFOR
\ENDFOR

\vspace{0.5em}
\textbf{// Phase 2: Joint training with causal error simulation}
\FOR{epoch =  $ E_w+1 $  \textbf{to}  $ E $ }
    \FOR{each batch of samples constructed from contiguous segments of length  $ 3H $ }
        \STATE // Segment:  $ [y_{t-2H}, \dots, y_{t+H-1}] $ 
        \STATE  $ X_{\text{ctx}} \gets [y_{t-2H}, \dots, y_{t-H-1}] $  \hfill // context input
        \STATE  $ X_{\text{in}} \gets [y_{t-H}, \dots, y_{t-1}] $  \hfill // current input
        \STATE  $ Y_{\text{hist}} \gets [y_{t-H}, \dots, y_{t-1}] $  \hfill // true values for error
        \STATE  $ Y_{\text{tgt}} \gets [y_t, \dots, y_{t+H-1}] $  \hfill // prediction target

        \STATE  $ \hat{Y}_{\text{ctx}} \gets f(X_{\text{ctx}}) $  \hfill // forecast for  $ Y_{\text{hist}} $ 
        \STATE  $ \epsilon \gets Y_{\text{hist}} - \hat{Y}_{\text{ctx}} $ 
        \STATE  $ \hat{Y}_{\text{base}} \gets f(X_{\text{in}}) $  \hfill // base forecast for  $ Y_{\text{tgt}} $ 
        \STATE  $ \Delta \gets g(\epsilon) $ 
        \STATE  $ \hat{Y}_{\text{final}} \gets \hat{Y}_{\text{base}} + \Delta $ 

        \STATE  $ \mathcal{L} \gets\mathcal{L}_{\text{acc}}(Y_{\text{tgt}}, \hat{Y}_{\text{final}} )$ 
        \STATE Update both  $ f $  and  $ g $  via  $ \nabla_{f,g} \mathcal{L} $ 
    \ENDFOR
\ENDFOR
\end{algorithmic}
\end{algorithm}
\subsection{Hyperparameters}
Regarding hyperparameters, for the base forecasting models, we adopt the hyperparameters from their official code repositories. For several datasets, such as AQShunyi, AQWAN, and EURUSD, the official codebases do not provide hyperparameters. Considering that these datasets are similar in size to the ETTm dataset, we use the hyperparameters corresponding to ETTm instead.

For our own method, the warm-up phase consists of three epochs. During the warm-up of the base model, the loss comprises two components: the base forecasting loss and the spectral flatness loss, with equal weights assigned to each. The training configurations, such as the learning rate, remain the same as those used for training the base model. During joint training, we train for 12 epochs with the AdamW optimizer. The remaining configurations, such as learning rate decay and patience, follow the default settings in official code repositories.

As for the error module, a bottleneck layer is first applied to reduce the error vector dimension; we set this dimension to 64.
\subsection{Details of shock and distribution shift in the main te}
Regarding Shock Injection: We uniformly injected 30 shocks across the entire dataset, ensuring consistent intervals between each shock. This ensures that shock events appear within the training, validation, and test sets. Shocks were applied to the normalized data. Each shock starts at its maximum amplitude of 3 and then linearly decays. The amplitude reaches zero at the 196th time step after the shock onset.

Regarding Distribution Shift: For the normalized dataset, let $y_t$
represent the data at time step $t$, and $L$ denote the total length of the dataset. We modified data points where $t>0.5L$ as $y_t=y_t+\frac{4(t-0.5L)}{L}$. This introduces a linear drift, causing a gradually increasing distribution shift throughout the latter half of the data. All experiments under distribution are conducted without RevIN.

\subsection{Computational Overhead.}
TEFL introduces modest computational and memory overhead during training, with negligible cost at inference. 
Specifically, in the joint training phase, each training sample is constructed from a contiguous segment of length  $ 3H $ , which is split into a context window and an input window. This requires doubling the batch size. Consequently, GPU memory consumption increases by a constant factor (typically less than 2 $ \times $ ), but this remains manageable on modern hardware.

Regarding training time, the only additional cost comes from the initial warm-up phase, which typically lasts for a small number of epochs (e.g., 3 in our experiments). The subsequent joint training phase has a per-epoch cost comparable to standard training. Thus, the total training time of TEFL is on the same order of magnitude as that of the standard training.

Critically, during deployment in rolling forecasting scenarios, TEFL incurs \textit{negligible overhead}. The historical residuals  $ \epsilon_{t-H:t-1} $  are naturally available from previous predictions, and the lightweight error module  $ g $  (e.g., a low-rank adapter) adds minimal latency. Therefore, TEFL offers its accuracy and robustness benefits with a highly favorable computation-accuracy trade-off.
\section{Analysis of using longer input window}
\begin{figure}[!h]
    \centering
    \includegraphics[width=0.95\linewidth]{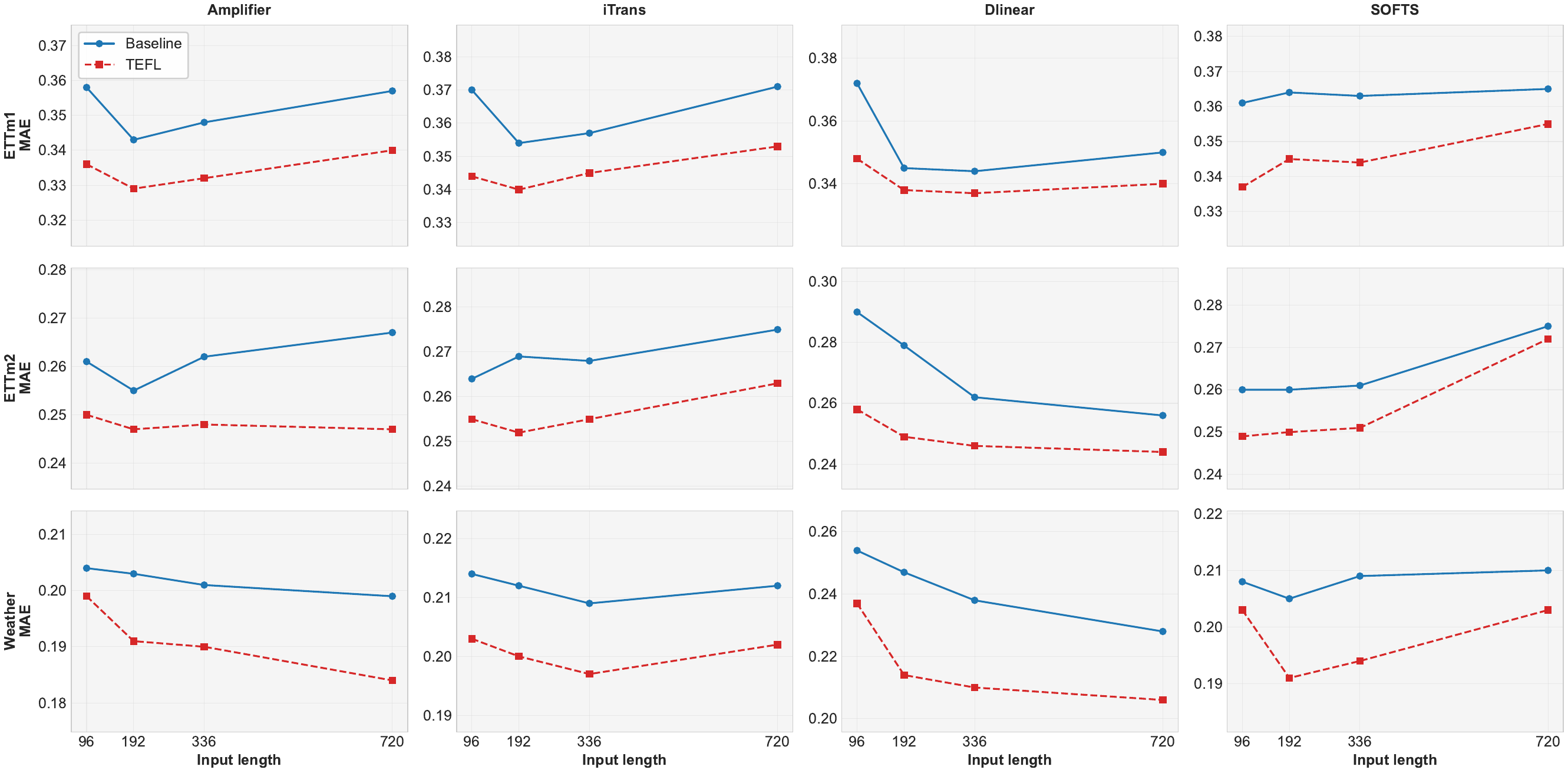}
    \caption{Comparison of Baseline and TEFL Methods for different input window size}
    \label{fig:longer}
\end{figure}
We also conducted experiments to investigate the effects of using longer input lengths on certain datasets. Specifically, for the ETTm1, ETTm2, and Weather datasets, we fixed the prediction horizon at 96 steps and sequentially extended the look-back window length from 96 to 192, 336, and 720. We repeated the experiments under these configurations, training models with both the original method and our TEFL method. The resulting testset MAE for each setting is plotted in Figure \ref{fig:longer}.

First, a key finding is that regardless of the input window length, the accuracy achieved by TEFL method is consistently better than that of the original method. Second, for different datasets and models, there appears to be a different optimal input window length. 

Finally, we would like to highlight that our method can utilize a longer historical context more effectively. For instance, when the input window size is 96, our method is somehow equivalent to extending the input size to 192 steps. This is because TEFL method first uses data from steps $t-192$ to $t-96$ to predict the segment from $t-96$ to $t$. The error from this prediction is then incorporated, along with the original data from $t-96$ to $t$, to generate the final forecast. The results of Figure \ref{fig:longer} show that our method can be viewed as a more efficient way to integrate longer historical information. This is supported by comparing the second blue point (the original method with a 192-step lookback window) with the first red point (our method with a 96-step lookback window) in each subplot of Figure \ref{fig:longer}. The comparison shows that in most cases, directly extending the input window of the original model from 96 to 192 steps does not yield better results than using our method with a 96-step lookback window.

\section{Discussion of training method}
\label{section:method}
\subsection{Discussion of spectral flatness}

Spectral flatness is a commonly used metric to evaluate whether a time series possesses structure. If the spectral flatness is close to 1, it indicates that the time series has little structure and is close to white noise. Therefore, we introduce a regularization term using spectral flatness of the prediction errors at consecutive time steps. This ensures that the errors of the base model exhibit a certain structure over time, allowing the subsequent error module to learn more effectively.

To specifically illustrate the temporal patterns corresponding to different levels of spectral flatness, four time series examples are presented in Figure \ref{fig:sf}. Each row in the figure corresponds to one time series, arranged in descending order of spectral flatness from top to bottom. Each row contains two subplots: the left subplot displays the original time series, while the right subplot shows its frequency-domain energy distribution after Fourier transform.

From top to bottom, the spectral flatness of the time series decreases progressively. The top row represents white noise, which possesses the highest spectral flatness. In contrast, the bottom row corresponds to a pure sine wave, where nearly all the energy is concentrated at a single frequency, resulting in a spectral flatness approaching zero.
This visualization clearly demonstrates the close relationship between spectral flatness and the presence of temporal structure in a signal. Consequently, during pre-training, applying regularization to minimize the spectral flatness of prediction errors effectively preserves a certain degree of temporal structure within the error sequence.
\begin{figure}[!h]
    \centering
    \includegraphics[width=0.95\linewidth]{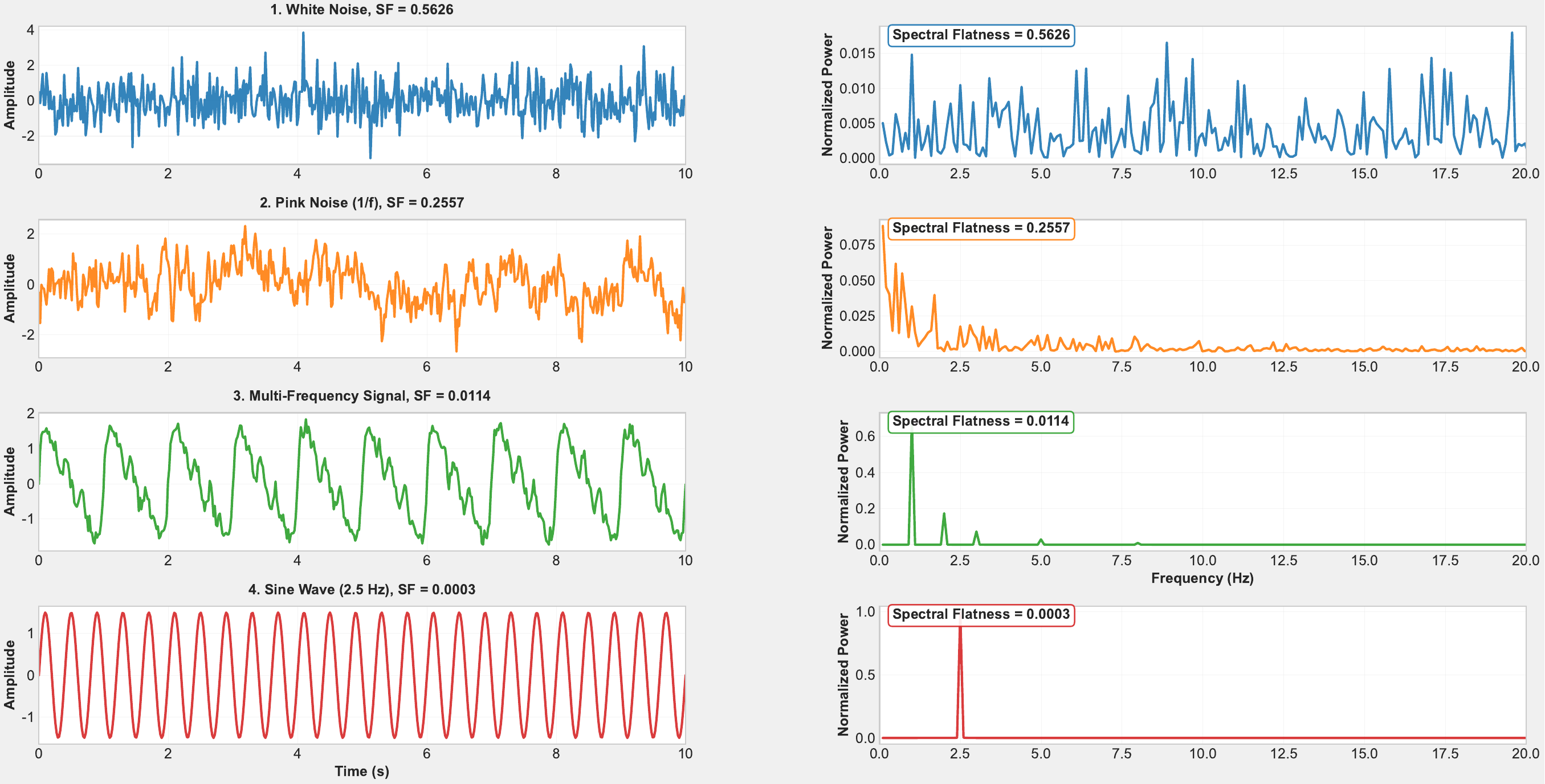}
    \caption{Examples about spectral flatness}
    \label{fig:sf}
\end{figure}
\subsection{Experiments using different training method}
We conducted additional experiments on the ETTm and Weather datasets to compare different training schemes. Specifically, we considered the following three schemes: (1) first training the base forecasting model until convergence, then fixing the base model and training only the error module; (2) jointly training the base model and the error module from scratch; and (3) removing the spectral flatness regularization from our training method. The results of these training methods are presented in Table \ref{tab:train_method}. In this table, the three schemes are labeled as Type 1, Type 2, and w/o SF (without spectral flatness). It can be observed that in most cases our method outperforms the other three variants. Moreover, Type 1 sometimes performs even worse than the original approach that does not incorporate error information. This may be because when the base model is trained to convergence, overfitting has already occurred, leading to a discrepancy of the error distributions between the training and test sets. Consequently, the error module trained solely on training set errors struggles to generalize to the test set. Meanwhile, the results of the w/o SF scheme also provide experimental evidence for the utility of the spectral flatness regularization term.
\begin{table}[!h]
\caption{Results of different training methods}
\label{tab:train_method}
\fontsize{8.5}{8.5}\selectfont
\renewcommand{\arraystretch}{1}
\setlength{\tabcolsep}{4pt}
\centering
\begin{tabular}{c|cccccccccccccc}
\toprule[1.5pt]
\multicolumn{3}{c}{Model}                                                  & \multicolumn{2}{c}{SOFTS}                                   & \multicolumn{2}{c}{iTrans}                                  & \multicolumn{2}{c}{Dlinear}                                 & \multicolumn{2}{c}{TimeFilter}                              & \multicolumn{2}{c}{Amplifier}                               & \multicolumn{2}{c}{AVG}                                     \\
\multicolumn{3}{c}{Metric}                                                 & MSE                          & MAE                          & MSE                          & MAE                          & MSE                          & MAE                          & MSE                          & MAE                          & MSE                          & MAE                          & MSE                          & MAE                          \\ \cmidrule[1pt]{1-15}
                           &                       & Baseline                  & 0.322                        & 0.361                        & 0.344                        & 0.370                        & 0.345                        & 0.372                        & 0.315                        & 0.356                        & 0.316                        & 0.358                        & 0.329                        & 0.363                        \\
                           &                       & Type 1                & 0.323                        & 0.364                        & 0.342                        & 0.380                        & 0.344                        & 0.376                        & 0.321                        & 0.361                        & 0.317                        & 0.359                        & 0.330                        & 0.368                        \\
\end{tabular}
\end{table}

\section{Discussion about historical errors}
\label{secction:error_define}
Figure \ref{fig:errors} illustrates three choices of historical errors examined in this study, using the prediction of the next 96 steps as an example. The i-th row of each matrix represents the prediction errors for the next 96 steps at time i. The matrix displays the case at $t=100$. Shaded areas represent unobservable errors, and red areas show the selected historical errors. The first subfigure corresponds to the error selection introduced in the main text. At time $t$, prediction errors for steps 1 to 96 ahead at time $t-96$ are selected. The second subfigure uses  the one-step-ahead prediction errors at times $t-96, t-95, ..., t-1$ as the input for error module. The third subfigure uses the prediction errors for time $t$ at times $t-96, t-95, ..., t-1$.

Intuitively, the choice in subfigure 3 may appear more reasonable. This is because all error calculations use the true value at time $t$, i.e., the most recently observed data, thus minimizing latency. However, what are the actual results?
Table \ref{tab:error_def} reports the experimental results for these three error selections. In the table, TEFL refers to the choice from the main text. Type2 and Type3 correspond to the error selections shown in the second and third subfigure, respectively.

As shown in Table \ref{tab:error_def}, our proposed TEFL method achieves the best performance in most cases. Only in a few long-term forecasting tasks with prediction horizons of 336 and 720 steps does Type3 outperform TEFL.

We can analyze this phenomenon as follows. First, it is important to clarify that our method can be regarded as an approach using historical prediction errors to estimate the current prediction errors. However, all errors defined by Type2 are one-step-ahead prediction errors, whereas the current prediction errors are multi-step errors for the next 1 to 96 steps. The distributions of these two types of errors are clearly different, making it challenging to estimate multi-step errors solely based on one-step errors.

As for Type3, it performs better in a few long-term forecasting cases (e.g., 336 and 720 steps) because it indeed uses the most recently observed data, resulting in lower latency, which is advantageous for long-term predictions. However, its drawback lies in that all errors are calculated based on the true value at time $t$, leading to higher volatility (e.g., when there is an obvious outlier at time $t$). In contrast, TEFL incorporates true values from multiple time steps $(t, t-1, ...)$, increasing the amount of information and reducing the impact of random fluctuations.
\begin{figure}[!h]
    \centering
    \includegraphics[width=0.95\linewidth]{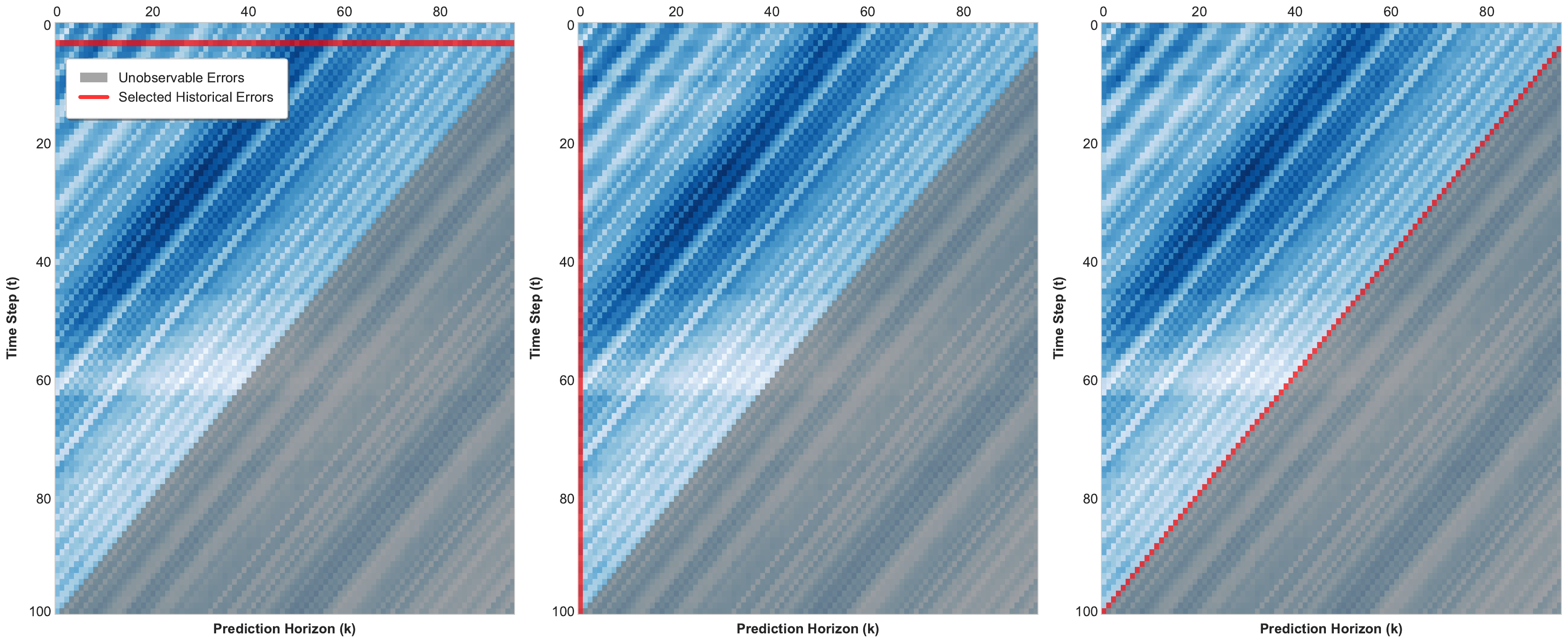}
    \caption{Different choices of historical errors}
    \label{fig:errors}
\end{figure}
\begin{figure}[!h]
    \centering
    \includegraphics[width=0.95\linewidth]{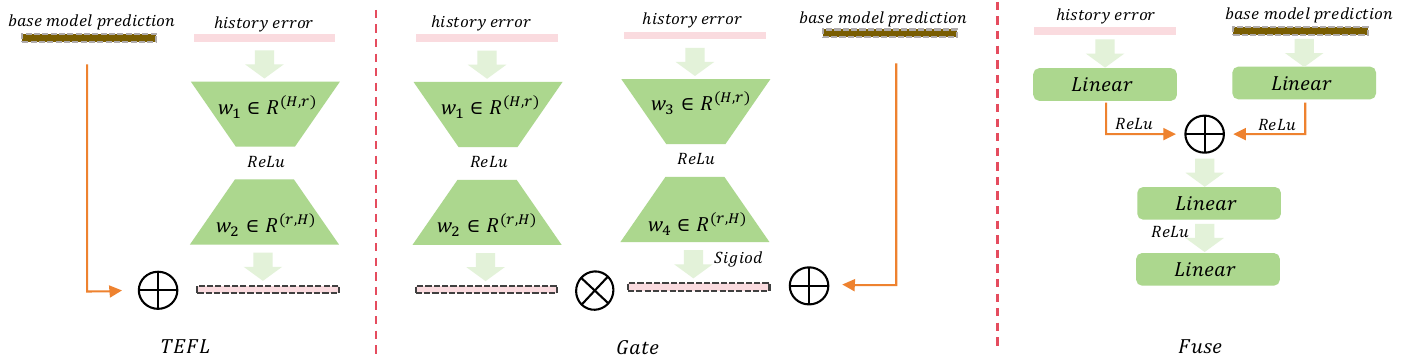}
    \caption{Different error modules}
    \label{fig:error_model}
\end{figure}
\begin{table}[!h]
\caption{Results of different error selection methods}
\label{tab:error_def}
\fontsize{8.5}{8.5}\selectfont
\renewcommand{\arraystretch}{1.1}
\setlength{\tabcolsep}{4pt}
\centering
\begin{tabular}{c|cccccccccccccc}
\toprule[1.5pt]
\multicolumn{3}{c}{Model}                                  & \multicolumn{2}{c}{SOFTS}                                   & \multicolumn{2}{c}{iTrans}                                  & \multicolumn{2}{c}{Dlinear}                                 & \multicolumn{2}{c}{TimeFilter}                              & \multicolumn{2}{c}{Amplifier}                               & \multicolumn{2}{c}{AVG}                                     \\
\multicolumn{3}{c}{Metric}                                 & MSE                          & MAE                          & MSE                          & MAE                          & MSE                          & MAE                          & MSE                          & MAE                          & MSE                          & MAE                          & MSE                          & MAE                          \\ \midrule[1pt]
                           &                       & Baseline   & 0.322                        & 0.361                        & 0.344                        & 0.370                        & 0.345                        & 0.372                        & 0.315                        & 0.356                        & 0.316                        & 0.358                        & 0.329                        & 0.363                        \\
                           &                       & Type2 & 0.323                        & 0.361                        & 0.339                        & 0.375                        & 0.333                        & 0.367                        & 0.318                        & 0.358                        & 0.313                        & 0.354                        & 0.325                        & 0.363                        \\
\end{tabular}
\end{table}
\section{Discussion about error module}
\label{section:err_model}
Figure \ref{fig:error_model} illustrates three different error module structures explored in this study. The first sub‑figure corresponds to the structure adopted in the main text of the paper. The second and third sub‑figures present two alternative designs. The second module design is called the "Gate" scheme, because, compared with the original error module of TEFL, it introduces a gating mechanism: the historical errors are passed through a feed‑forward network (FFN), multiplied by a confidence score (generated by the another FFN), and then added to the base model’s predictions. This design additionally considers the confidence of error correction; Intuitively, it should learn when to rely more on historical errors to correct the base model’s predictions and when to trust the base model’s predictions more. The design in the third sub‑figure is called the "Fuse" method. It projects both the errors and the base model predictions, merges them, and then passes the result through an FFN to obtain the final output, thereby allowing the interaction between historical errors and the base model predictions to be considered. 

Table \ref{tab:error_module} reports the experimental results of these three schemes on the ETTm and Weather datasets. It can be observed that these two variants do not outperform the original, simplest TEFL method in most cases. Especially for the "Fuse" method, although it possesses strong expressive power in theory, it even performs worse than the baseline method without an error module on the ETTm dataset, which may be caused by overfitting due to its excessive number of parameters.

\begin{table}[!h]
\caption{Results of different error modules}
\label{tab:error_module}
\fontsize{8.5}{8.5}\selectfont
\renewcommand{\arraystretch}{1.1}
\setlength{\tabcolsep}{4pt}
\centering
\begin{tabular}{c|cccccccccccccc}
\toprule[1.5pt]
\multicolumn{3}{c}{Model}                                 & \multicolumn{2}{c}{SOFTS}                                   & \multicolumn{2}{c}{iTrans}                                  & \multicolumn{2}{c}{Dlinear}                                 & \multicolumn{2}{c}{TimeFilter}                              & \multicolumn{2}{c}{Amplifier}                               & \multicolumn{2}{c}{AVG}                                     \\
\multicolumn{3}{c}{Metric}                                & MSE                          & MAE                          & MSE                          & MAE                          & MSE                          & MAE                          & MSE                          & MAE                          & MSE                          & MAE                          & MSE                          & MAE                          \\ \midrule[1pt]
                           &                       & Baseline  & 0.322                        & 0.361                        & 0.344                        & 0.370                        & 0.345                        & 0.372                        & 0.315                        & 0.356                        & 0.316                        & 0.358                        & 0.329                        & 0.363                        \\
                           &                       & Gate & 0.320                        & 0.345                        & 0.320                        & 0.349                        & 0.328                        & 0.349                        & 0.310                        & 0.337                        & 0.326                        & 0.350                        & 0.321                        & 0.346                        \\
\end{tabular}
\end{table}
\section{Reconciling Batch Training with Rolling Forecasting}

As highlighted in the introduction, a fundamental mismatch exists between standard training protocols and real-world deployment: models are typically trained on randomly shuffled segments using batch optimization, whereas in practice they operate in a rolling fashion—making predictions step-by-step (or horizon-by-horizon) while sequentially observing ground truth and accumulating prediction residuals.

One might question why we retain batch training in our method despite acknowledging this mismatch. The reason is pragmatic yet principled: completely abandoning batch training would sacrifice scalability, convergence stability, and compatibility with modern deep architectures. Instead, TEFL seeks to \textit{bridge} the gap by \textit{mimicking} the rolling feedback loop within a batch-compatible framework.

Specifically, during training, we construct each sample as a contiguous subsequence of length   $  3H  $  , where: the first   $  H  $   steps serve as the historical context, the next   $  H  $   steps are used to compute a "simulated" residual history (by applying the base model   $  f  $   to predict them from the prior   $  H  $   steps), and the final   $  H  $   steps constitute the true prediction target.

The error feedback module   $  g  $   then takes the simulated residuals from the middle segment as input to correct the base prediction for the last   $  H  $   steps. Crucially, this simulation is performed \textit{within each sample independently}, without requiring cross-sample sequential dependencies. Thus, samples can be shuffled and processed in parallel, preserving the efficiency of batch training.

This design approximates the rolling scenario: just as a deployed system uses past errors to inform future forecasts, our training samples embed a short “look-back-and-correct” cycle internally. While it does not fully replicate the long-term error propagation seen in true rolling evaluation (where errors can build up over hundreds of steps), it provides sufficient signal for the model to learn how to interpret and utilize recent residuals—a capability that consistently leads to improved performance at test time, as shown in our experiments.

In essence, TEFL does not fully eliminate the train–test mismatch, but it \textit{injects a causal, rolling-like signal into batch training}. This compromise enables both practical training efficiency and meaningful adaptation to recent prediction behavior. This gives us the best of both worlds: the speed and simplicity of batch training, plus the ability to adapt based on recent prediction behavior—something standard batch methods miss, and full online methods achieve only at much higher computational cost.
\section{Full results}
\subsection{Full results of main experiments}
Full results of main experiments are provided in Table \ref{tab:full1} and \ref{tab:full2}.
\begin{table}[]
\caption{Full results I}
\label{tab:full1}
\fontsize{9}{9}\selectfont
\renewcommand{\arraystretch}{1.5}
\setlength{\tabcolsep}{1.3pt}
\centering
\begin{tabular}{c|cccccccccccccccc}
\toprule[1.5pt]
\multicolumn{3}{c}{Model}                                    & \multicolumn{2}{c}{SOFTS}                                   & \multicolumn{2}{c}{iTrans}                                  & \multicolumn{2}{c}{Dlinear}                                 & \multicolumn{2}{c}{TimeFilter}                              & \multicolumn{2}{c}{Amplifier}                               & \multicolumn{2}{c}{Avg}                                     & \multicolumn{2}{c}{Imp}                            \\
\multicolumn{3}{c}{Metric}                                   & MSE                          & MAE                          & MSE                          & MAE                          & MSE                          & MAE                          & MSE                          & MAE                          & MSE                          & MAE                          & MSE                          & MAE                          & MSE                      & MAE                     \\ \bottomrule[1pt]
                              &                       & Baseline  & 0.322                        & 0.361                        & 0.344                        & 0.370                        & 0.345                        & 0.372                        & 0.315                        & 0.356                        & 0.316                        & 0.358                        & 0.329                        & 0.363                        &                          &                         \\
\end{tabular}
\end{table}

\begin{table}[]
\fontsize{9}{9}\selectfont
\renewcommand{\arraystretch}{1.5}
\setlength{\tabcolsep}{1.3pt}
\centering
\caption{Full results II}
\label{tab:full2}
\begin{tabular}{c|cccccccccccccccc}
\toprule[1.5pt]
\multicolumn{3}{c}{Model}                                 & \multicolumn{2}{c}{SOFTS}                                   & \multicolumn{2}{c}{iTrans}                                  & \multicolumn{2}{c}{Dlinear}                                 & \multicolumn{2}{c}{TimeFilter}                              & \multicolumn{2}{c}{Amplifier}                               & \multicolumn{2}{c}{Avg}                                     & \multicolumn{2}{c}{Imp}                             \\
\multicolumn{3}{c}{Metric}                                & MSE                          & MAE                          & MSE                          & MAE                          & MSE                          & MAE                          & MSE                          & MAE                          & MSE                          & MAE                          & MSE                          & MAE                          & MSE                      & MAE                      \\ \midrule[1pt]
                           &                       & Baseline  & 0.814                        & 0.489                        & 0.832                        & 0.498                        & 0.784                        & 0.512                        & 0.803                        & 0.486                        & 0.788                        & 0.483                        & 0.804                        & 0.494                        &                          &                          \\
\end{tabular}
\end{table}
\subsection{Full results of datasets with shocks}
The complete experimental results on the shock-injected datasets are reported in Table \ref{tab:full shock}. It can be observed that our TEFL method achieves a significant performance improvement on the dataset with shocks. Furthermore, the extent of this improvement is greater than that observed on the original dataset, which demonstrates the robustness and superior adaptability of TEFL to shock events.
\begin{table}[!h]
\fontsize{9}{9}\selectfont
\renewcommand{\arraystretch}{1.5}
\setlength{\tabcolsep}{1.3pt}
\centering
\caption{Full results for dataset with shocks}
\label{tab:full shock}
\begin{tabular}{c|cccccccccccccccc}
\toprule[1.5pt]
\multicolumn{3}{c}{Model}                                & \multicolumn{2}{c}{Amplifier}                               & \multicolumn{2}{c}{iTrans}                                  & \multicolumn{2}{c}{Dlinear}                                 & \multicolumn{2}{c}{SOFTS}                                   & \multicolumn{2}{c}{TimeFilter}                              & \multicolumn{2}{c}{Avg}                                     & \multicolumn{2}{c}{Imp}                             \\
\multicolumn{3}{c}{Metric}                               & MSE                          & MAE                          & MSE                          & MAE                          & MSE                          & MAE                          & MSE                          & MAE                          & MSE                          & MAE                          & MSE                          & MAE                          & MSE                      & MAE                      \\ \midrule[1pt]
                          &                       & Baseline  & 0.543                        & 0.463                        & 0.653                        & 0.514                        & 0.583                        & 0.490                        & 0.531                        & 0.453                        & 0.497                        & 0.432                        & 0.561                        & 0.470                        &                          &                          \\
                          & \multirow{-2}{*}{96}  & TEFL & {\color[HTML]{FF0000} 0.498} & {\color[HTML]{FF0000} 0.428} & {\color[HTML]{FF0000} 0.573} & {\color[HTML]{FF0000} 0.465} & {\color[HTML]{FF0000} 0.523} & {\color[HTML]{FF0000} 0.442} & {\color[HTML]{FF0000} 0.475} & {\color[HTML]{FF0000} 0.408} & {\color[HTML]{FF0000} 0.459} & {\color[HTML]{FF0000} 0.397} & {\color[HTML]{FF0000} 0.506} & {\color[HTML]{FF0000} 0.428} & \multirow{-2}{*}{9.9\%}  & \multirow{-2}{*}{8.9\%}  \\
                          &                       & Baseline  & 0.541                        & 0.459                        & 0.657                        & 0.512                        & 0.686                        & 0.543                        & 0.648                        & 0.511                        & 0.600                        & 0.487                        & 0.626                        & 0.503                        &                          &                          \\
                          & \multirow{-2}{*}{192} & TEFL & {\color[HTML]{FF0000} 0.490} & {\color[HTML]{FF0000} 0.423} & {\color[HTML]{FF0000} 0.589} & {\color[HTML]{FF0000} 0.472} & {\color[HTML]{FF0000} 0.594} & {\color[HTML]{FF0000} 0.478} & {\color[HTML]{FF0000} 0.594} & {\color[HTML]{FF0000} 0.478} & {\color[HTML]{FF0000} 0.548} & {\color[HTML]{FF0000} 0.447} & {\color[HTML]{FF0000} 0.563} & {\color[HTML]{FF0000} 0.460} & \multirow{-2}{*}{10.1\%} & \multirow{-2}{*}{8.5\%}  \\
                          &                       & Baseline  & 0.706                        & 0.544                        & 0.709                        & 0.542                        & 0.743                        & 0.579                        & 0.715                        & 0.543                        & 0.648                        & 0.520                        & 0.704                        & 0.546                        &                          &                          \\
                          & \multirow{-2}{*}{336} & TEFL & {\color[HTML]{FF0000} 0.632} & {\color[HTML]{FF0000} 0.498} & {\color[HTML]{FF0000} 0.641} & {\color[HTML]{FF0000} 0.504} & {\color[HTML]{FF0000} 0.646} & {\color[HTML]{FF0000} 0.510} & {\color[HTML]{FF0000} 0.660} & {\color[HTML]{FF0000} 0.520} & {\color[HTML]{FF0000} 0.605} & {\color[HTML]{FF0000} 0.481} & {\color[HTML]{FF0000} 0.637} & {\color[HTML]{FF0000} 0.503} & \multirow{-2}{*}{9.5\%}  & \multirow{-2}{*}{7.9\%}  \\
                          &                       & Baseline  & 0.810                        & 0.598                        & 0.798                        & 0.588                        & 0.832                        & 0.629                        & 0.811                        & 0.598                        & 0.754                        & 0.573                        & 0.801                        & 0.597                        &                          &                          \\
\multirow{-8}{*}{ETTm1}   & \multirow{-2}{*}{720} & TEFL & {\color[HTML]{FF0000} 0.698} & {\color[HTML]{FF0000} 0.546} & {\color[HTML]{FF0000} 0.733} & {\color[HTML]{FF0000} 0.552} & {\color[HTML]{FF0000} 0.747} & {\color[HTML]{FF0000} 0.568} & {\color[HTML]{FF0000} 0.742} & {\color[HTML]{FF0000} 0.566} & {\color[HTML]{FF0000} 0.695} & {\color[HTML]{FF0000} 0.531} & {\color[HTML]{FF0000} 0.723} & {\color[HTML]{FF0000} 0.553} & \multirow{-2}{*}{9.7\%}  & \multirow{-2}{*}{7.4\%}  \\ \midrule[0.5pt]
                          &                       & Baseline  & 0.340                        & 0.335                        & 0.326                        & 0.334                        & 0.414                        & 0.430                        & 0.317                        & 0.325                        & 0.315                        & 0.322                        & 0.342                        & 0.349                        &                          &                          \\
                          & \multirow{-2}{*}{96}  & TEFL & {\color[HTML]{FF0000} 0.318} & {\color[HTML]{FF0000} 0.316} & {\color[HTML]{FF0000} 0.311} & {\color[HTML]{FF0000} 0.316} & {\color[HTML]{FF0000} 0.336} & {\color[HTML]{FF0000} 0.343} & {\color[HTML]{FF0000} 0.303} & {\color[HTML]{FF0000} 0.305} & {\color[HTML]{FF0000} 0.302} & {\color[HTML]{FF0000} 0.306} & {\color[HTML]{FF0000} 0.314} & {\color[HTML]{FF0000} 0.317} & \multirow{-2}{*}{8.3\%}  & \multirow{-2}{*}{9.1\%}  \\
                          &                       & Baseline  & 0.492                        & 0.417                        & 0.457                        & 0.407                        & 0.639                        & 0.554                        & 0.460                        & 0.404                        & 0.455                        & 0.400                        & 0.501                        & 0.436                        &                          &                          \\
                          & \multirow{-2}{*}{192} & TEFL & {\color[HTML]{FF0000} 0.455} & {\color[HTML]{FF0000} 0.394} & {\color[HTML]{FF0000} 0.432} & {\color[HTML]{FF0000} 0.383} & {\color[HTML]{FF0000} 0.460} & {\color[HTML]{FF0000} 0.413} & {\color[HTML]{FF0000} 0.427} & {\color[HTML]{FF0000} 0.375} & {\color[HTML]{FF0000} 0.428} & {\color[HTML]{FF0000} 0.375} & {\color[HTML]{FF0000} 0.440} & {\color[HTML]{FF0000} 0.388} & \multirow{-2}{*}{12.1\%} & \multirow{-2}{*}{11.0\%} \\
                          &                       & Baseline  & 0.643                        & 0.493                        & 0.559                        & 0.463                        & 0.823                        & 0.650                        & 0.557                        & 0.457                        & 0.554                        & 0.457                        & 0.627                        & 0.504                        &                          &                          \\
                          & \multirow{-2}{*}{336} & TEFL & {\color[HTML]{FF0000} 0.556} & {\color[HTML]{FF0000} 0.455} & {\color[HTML]{FF0000} 0.520} & {\color[HTML]{FF0000} 0.435} & {\color[HTML]{FF0000} 0.550} & {\color[HTML]{FF0000} 0.474} & {\color[HTML]{FF0000} 0.510} & {\color[HTML]{FF0000} 0.423} & {\color[HTML]{FF0000} 0.512} & {\color[HTML]{FF0000} 0.428} & {\color[HTML]{FF0000} 0.530} & {\color[HTML]{FF0000} 0.443} & \multirow{-2}{*}{15.6\%} & \multirow{-2}{*}{12.0\%} \\
                          &                       & Baseline  & 0.777                        & 0.560                        & 0.701                        & 0.537                        & 1.058                        & 0.757                        & {\color[HTML]{FF0000} 0.692} & {\color[HTML]{FF0000} 0.530} & 0.707                        & 0.536                        & 0.787                        & 0.584                        &                          &                          \\
\multirow{-8}{*}{ETTm2}   & \multirow{-2}{*}{720} & TEFL & {\color[HTML]{FF0000} 0.691} & {\color[HTML]{FF0000} 0.535} & {\color[HTML]{FF0000} 0.657} & {\color[HTML]{FF0000} 0.516} & {\color[HTML]{FF0000} 0.698} & {\color[HTML]{FF0000} 0.568} & 0.643                        & 0.503                        & {\color[HTML]{FF0000} 0.667} & {\color[HTML]{FF0000} 0.520} & {\color[HTML]{FF0000} 0.671} & {\color[HTML]{FF0000} 0.528} & \multirow{-2}{*}{14.7\%} & \multirow{-2}{*}{9.5\%}  \\\midrule[0.5pt]
                          &                       & Baseline  & 0.414                        & 0.334                        & 0.375                        & 0.307                        & 0.449                        & 0.410                        & 0.368                        & 0.299                        & 0.353                        & 0.295                        & 0.392                        & 0.329                        &                          &                          \\
                          & \multirow{-2}{*}{96}  & TEFL & {\color[HTML]{FF0000} 0.380} & {\color[HTML]{FF0000} 0.307} & {\color[HTML]{FF0000} 0.354} & {\color[HTML]{FF0000} 0.279} & {\color[HTML]{FF0000} 0.432} & {\color[HTML]{FF0000} 0.348} & {\color[HTML]{FF0000} 0.348} & {\color[HTML]{FF0000} 0.272} & {\color[HTML]{FF0000} 0.335} & {\color[HTML]{FF0000} 0.267} & {\color[HTML]{FF0000} 0.370} & {\color[HTML]{FF0000} 0.295} & \multirow{-2}{*}{5.6\%}  & \multirow{-2}{*}{10.4\%} \\
                          &                       & Baseline  & 0.568                        & 0.423                        & 0.502                        & 0.386                        & 0.584                        & 0.506                        & 0.500                        & 0.387                        & 0.476                        & 0.375                        & 0.526                        & 0.415                        &                          &                          \\
                          & \multirow{-2}{*}{192} & TEFL & {\color[HTML]{FF0000} 0.508} & {\color[HTML]{FF0000} 0.384} & {\color[HTML]{FF0000} 0.483} & {\color[HTML]{FF0000} 0.365} & {\color[HTML]{FF0000} 0.516} & {\color[HTML]{FF0000} 0.403} & {\color[HTML]{FF0000} 0.466} & {\color[HTML]{FF0000} 0.350} & {\color[HTML]{FF0000} 0.456} & {\color[HTML]{FF0000} 0.345} & {\color[HTML]{FF0000} 0.486} & {\color[HTML]{FF0000} 0.369} & \multirow{-2}{*}{7.7\%}  & \multirow{-2}{*}{11.1\%} \\
                          &                       & Baseline  & 0.687                        & 0.484                        & 0.609                        & 0.447                        & 0.682                        & 0.568                        & 0.608                        & 0.448                        & 0.583                        & 0.438                        & 0.634                        & 0.477                        &                          &                          \\
                          & \multirow{-2}{*}{336} & TEFL & {\color[HTML]{FF0000} 0.600} & {\color[HTML]{FF0000} 0.440} & {\color[HTML]{FF0000} 0.581} & {\color[HTML]{FF0000} 0.426} & {\color[HTML]{FF0000} 0.590} & {\color[HTML]{FF0000} 0.453} & {\color[HTML]{FF0000} 0.556} & {\color[HTML]{FF0000} 0.408} & {\color[HTML]{FF0000} 0.563} & {\color[HTML]{FF0000} 0.417} & {\color[HTML]{FF0000} 0.578} & {\color[HTML]{FF0000} 0.429} & \multirow{-2}{*}{8.8\%}  & \multirow{-2}{*}{10.2\%} \\
                          &                       & Baseline  & 0.815                        & 0.555                        & 0.709                        & 0.506                        & 0.761                        & 0.625                        & 0.702                        & 0.503                        & 0.674                        & 0.493                        & 0.732                        & 0.536                        &                          &                          \\
\multirow{-8}{*}{Weather} & \multirow{-2}{*}{720} & TEFL & {\color[HTML]{FF0000} 0.668} & {\color[HTML]{FF0000} 0.486} & {\color[HTML]{FF0000} 0.669} & {\color[HTML]{FF0000} 0.484} & {\color[HTML]{FF0000} 0.659} & {\color[HTML]{FF0000} 0.511} & {\color[HTML]{FF0000} 0.633} & {\color[HTML]{FF0000} 0.459} & {\color[HTML]{FF0000} 0.626} & {\color[HTML]{FF0000} 0.456} & {\color[HTML]{FF0000} 0.651} & {\color[HTML]{FF0000} 0.479} & \multirow{-2}{*}{11.1\%} & \multirow{-2}{*}{10.7\%}\\ \bottomrule[1.5pt]
\end{tabular}
\end{table}
\subsection{Full results for datasets with distribution}
Regarding the dataset with distribution shift, the complete experimental results are provided in Table \ref{tab:full ds}. From this table, it can be observed that our method achieves a significant improvement in prediction accuracy under distribution shift, demonstrating its strong adaptability to such conditions. This finding aligns with the conclusions presented in the main text.
\begin{table}[!h]
\fontsize{9}{9}\selectfont
\renewcommand{\arraystretch}{1.5}
\setlength{\tabcolsep}{1.3pt}
\centering
\caption{Full results for datasets with distribution shift}

\label{tab:full ds}
\begin{tabular}{c|cccccccccccccccc}
\toprule[1.5pt]
\multicolumn{3}{c}{Model}                                & \multicolumn{2}{c}{Amplifier}                               & \multicolumn{2}{c}{iTrans}                                  & \multicolumn{2}{c}{Dlinear}                                 & \multicolumn{2}{c}{SOFTS}                                   & \multicolumn{2}{c}{TimeFilter}                              & \multicolumn{2}{c}{Avg}                                     & \multicolumn{2}{c}{Imp}                             \\
\multicolumn{3}{c}{Metric}                               & MSE                          & MAE                          & MSE                          & MAE                          & MSE                          & MAE                          & MSE                          & MAE                          & MSE                          & MAE                          & MSE                          & MAE                          & MSE                      & MAE                      \\ \midrule[1pt]
                          &                       & Baseline  & 0.409                        & 0.422                        & 0.456                        & 0.458                        & 0.421                        & 0.453                        & 0.513                        & 0.497                        & 0.535                        & 0.527                        & 0.467                        & 0.471                        &                          &                          \\
                          & \multirow{-2}{*}{96}  & TEFL & {\color[HTML]{FF0000} 0.368} & {\color[HTML]{FF0000} 0.394} & {\color[HTML]{FF0000} 0.441} & {\color[HTML]{FF0000} 0.448} & {\color[HTML]{FF0000} 0.327} & {\color[HTML]{FF0000} 0.358} & {\color[HTML]{FF0000} 0.425} & {\color[HTML]{FF0000} 0.436} & {\color[HTML]{FF0000} 0.512} & {\color[HTML]{FF0000} 0.481} & {\color[HTML]{FF0000} 0.414} & {\color[HTML]{FF0000} 0.423} & \multirow{-2}{*}{11.2\%} & \multirow{-2}{*}{10.2\%} \\
                          &                       & Baseline  & 0.424                        & 0.428                        & 0.493                        & 0.482                        & 0.496                        & 0.505                        & 0.483                        & 0.477                        & 0.558                        & 0.547                        & 0.491                        & 0.488                        &                          &                          \\
                          & \multirow{-2}{*}{192} & TEFL & {\color[HTML]{FF0000} 0.392} & {\color[HTML]{FF0000} 0.401} & {\color[HTML]{FF0000} 0.429} & {\color[HTML]{FF0000} 0.443} & {\color[HTML]{FF0000} 0.393} & {\color[HTML]{FF0000} 0.409} & {\color[HTML]{FF0000} 0.454} & {\color[HTML]{FF0000} 0.452} & {\color[HTML]{FF0000} 0.475} & {\color[HTML]{FF0000} 0.466} & {\color[HTML]{FF0000} 0.429} & {\color[HTML]{FF0000} 0.434} & \multirow{-2}{*}{12.6\%} & \multirow{-2}{*}{11.0\%} \\
                          &                       & Baseline  & 0.477                        & 0.466                        & 0.519                        & 0.500                        & 0.598                        & 0.578                        & 0.573                        & 0.532                        & 0.788                        & 0.684                        & 0.591                        & 0.552                        &                          &                          \\
                          & \multirow{-2}{*}{336} & TEFL & {\color[HTML]{FF0000} 0.438} & {\color[HTML]{FF0000} 0.434} & {\color[HTML]{FF0000} 0.476} & {\color[HTML]{FF0000} 0.476} & {\color[HTML]{FF0000} 0.441} & {\color[HTML]{FF0000} 0.447} & {\color[HTML]{FF0000} 0.473} & {\color[HTML]{FF0000} 0.471} & {\color[HTML]{FF0000} 0.537} & {\color[HTML]{FF0000} 0.512} & {\color[HTML]{FF0000} 0.473} & {\color[HTML]{FF0000} 0.468} & \multirow{-2}{*}{19.9\%} & \multirow{-2}{*}{15.2\%} \\
                          &                       & Baseline  & 0.579                        & 0.519                        & 0.657                        & 0.594                        & 0.840                        & 0.724                        & 0.768                        & 0.650                        & 0.765                        & 0.654                        & 0.722                        & 0.628                        &                          &                          \\
\multirow{-8}{*}{ETTm1}   & \multirow{-2}{*}{720} & TEFL & {\color[HTML]{FF0000} 0.544} & {\color[HTML]{FF0000} 0.486} & {\color[HTML]{FF0000} 0.624} & {\color[HTML]{FF0000} 0.574} & {\color[HTML]{FF0000} 0.572} & {\color[HTML]{FF0000} 0.538} & {\color[HTML]{FF0000} 0.581} & {\color[HTML]{FF0000} 0.535} & {\color[HTML]{FF0000} 0.610} & {\color[HTML]{FF0000} 0.551} & {\color[HTML]{FF0000} 0.586} & {\color[HTML]{FF0000} 0.537} & \multirow{-2}{*}{18.8\%} & \multirow{-2}{*}{14.5\%} \\\midrule[0.5pt]
                          &                       & Baseline  & 0.193                        & 0.292                        & 0.186                        & 0.287                        & 0.194                        & 0.299                        & 0.192                        & 0.295                        & 0.200                        & 0.302                        & 0.193                        & 0.295                        &                          &                          \\
                          & \multirow{-2}{*}{96}  & TEFL & {\color[HTML]{FF0000} 0.181} & {\color[HTML]{FF0000} 0.268} & {\color[HTML]{FF0000} 0.178} & {\color[HTML]{FF0000} 0.271} & {\color[HTML]{FF0000} 0.176} & {\color[HTML]{FF0000} 0.263} & {\color[HTML]{FF0000} 0.185} & {\color[HTML]{FF0000} 0.283} & {\color[HTML]{FF0000} 0.176} & {\color[HTML]{FF0000} 0.265} & {\color[HTML]{FF0000} 0.179} & {\color[HTML]{FF0000} 0.270} & \multirow{-2}{*}{7.2\%}  & \multirow{-2}{*}{8.6\%}  \\
                          &                       & Baseline  & 0.259                        & 0.343                        & 0.253                        & 0.340                        & 0.293                        & 0.387                        & 0.274                        & 0.363                        & 0.309                        & 0.375                        & 0.278                        & 0.362                        &                          &                          \\
                          & \multirow{-2}{*}{192} & TEFL & {\color[HTML]{FF0000} 0.248} & {\color[HTML]{FF0000} 0.311} & {\color[HTML]{FF0000} 0.234} & {\color[HTML]{FF0000} 0.315} & {\color[HTML]{FF0000} 0.237} & {\color[HTML]{FF0000} 0.307} & {\color[HTML]{FF0000} 0.257} & {\color[HTML]{FF0000} 0.340} & {\color[HTML]{FF0000} 0.242} & {\color[HTML]{FF0000} 0.319} & {\color[HTML]{FF0000} 0.243} & {\color[HTML]{FF0000} 0.318} & \multirow{-2}{*}{12.3\%} & \multirow{-2}{*}{12.0\%} \\
                          &                       & Baseline  & 0.347                        & 0.409                        & 0.330                        & 0.400                        & 0.378                        & 0.451                        & 0.404                        & 0.451                        & 0.425                        & 0.455                        & 0.377                        & 0.433                        &                          &                          \\
                          & \multirow{-2}{*}{336} & TEFL & {\color[HTML]{FF0000} 0.298} & {\color[HTML]{FF0000} 0.353} & {\color[HTML]{FF0000} 0.286} & {\color[HTML]{FF0000} 0.352} & {\color[HTML]{FF0000} 0.294} & {\color[HTML]{FF0000} 0.349} & {\color[HTML]{FF0000} 0.304} & {\color[HTML]{FF0000} 0.370} & {\color[HTML]{FF0000} 0.302} & {\color[HTML]{FF0000} 0.367} & {\color[HTML]{FF0000} 0.297} & {\color[HTML]{FF0000} 0.358} & \multirow{-2}{*}{21.3\%} & \multirow{-2}{*}{17.3\%} \\
                          &                       & Baseline  & 0.569                        & 0.523                        & 0.472                        & 0.511                        & 0.566                        & 0.576                        & 0.800                        & 0.628                        & 0.558                        & 0.516                        & 0.593                        & 0.551                        &                          &                          \\
\multirow{-8}{*}{ETTm2}   & \multirow{-2}{*}{720} & TEFL & {\color[HTML]{FF0000} 0.406} & {\color[HTML]{FF0000} 0.429} & {\color[HTML]{FF0000} 0.396} & {\color[HTML]{FF0000} 0.438} & {\color[HTML]{FF0000} 0.399} & {\color[HTML]{FF0000} 0.429} & {\color[HTML]{FF0000} 0.417} & {\color[HTML]{FF0000} 0.466} & {\color[HTML]{FF0000} 0.389} & {\color[HTML]{FF0000} 0.426} & {\color[HTML]{FF0000} 0.401} & {\color[HTML]{FF0000} 0.438} & \multirow{-2}{*}{32.3\%} & \multirow{-2}{*}{20.6\%} \\ \midrule[0.5pt]
                          &                       & Baseline  & 0.281                        & 0.351                        & 0.232                        & 0.291                        & 0.357                        & 0.443                        & 0.239                        & 0.301                        & 0.253                        & 0.326                        & 0.273                        & 0.342                        &                          &                          \\
                          & \multirow{-2}{*}{96}  & TEFL & {\color[HTML]{FF0000} 0.219} & {\color[HTML]{FF0000} 0.296} & {\color[HTML]{FF0000} 0.206} & {\color[HTML]{FF0000} 0.262} & {\color[HTML]{FF0000} 0.217} & {\color[HTML]{FF0000} 0.271} & {\color[HTML]{FF0000} 0.246} & {\color[HTML]{FF0000} 0.308} & {\color[HTML]{FF0000} 0.224} & {\color[HTML]{FF0000} 0.296} & {\color[HTML]{FF0000} 0.222} & {\color[HTML]{FF0000} 0.287} & \multirow{-2}{*}{18.4\%} & \multirow{-2}{*}{16.3\%} \\
                          &                       & Baseline  & 0.362                        & 0.412                        & 0.329                        & 0.367                        & 0.479                        & 0.522                        & 0.417                        & 0.435                        & 0.322                        & 0.388                        & 0.382                        & 0.425                        &                          &                          \\
                          & \multirow{-2}{*}{192} & TEFL & {\color[HTML]{FF0000} 0.237} & {\color[HTML]{FF0000} 0.291} & {\color[HTML]{FF0000} 0.328} & {\color[HTML]{FF0000} 0.382} & {\color[HTML]{FF0000} 0.254} & {\color[HTML]{FF0000} 0.305} & {\color[HTML]{FF0000} 0.397} & {\color[HTML]{FF0000} 0.429} & {\color[HTML]{FF0000} 0.282} & {\color[HTML]{FF0000} 0.346} & {\color[HTML]{FF0000} 0.299} & {\color[HTML]{FF0000} 0.351} & \multirow{-2}{*}{21.6\%} & \multirow{-2}{*}{17.5\%} \\
                          &                       & Baseline  & 0.402                        & 0.438                        & 0.574                        & 0.532                        & 0.604                        & 0.602                        & 0.587                        & 0.532                        & 0.404                        & 0.441                        & 0.514                        & 0.509                        &                          &                          \\
                          & \multirow{-2}{*}{336} & TEFL & {\color[HTML]{FF0000} 0.286} & {\color[HTML]{FF0000} 0.323} & {\color[HTML]{FF0000} 0.460} & {\color[HTML]{FF0000} 0.469} & {\color[HTML]{FF0000} 0.313} & {\color[HTML]{FF0000} 0.362} & {\color[HTML]{FF0000} 0.418} & {\color[HTML]{FF0000} 0.433} & {\color[HTML]{FF0000} 0.365} & {\color[HTML]{FF0000} 0.409} & {\color[HTML]{FF0000} 0.368} & {\color[HTML]{FF0000} 0.399} & \multirow{-2}{*}{28.4\%} & \multirow{-2}{*}{21.6\%} \\
                          &                       & Baseline  & 0.717                        & 0.601                        & 0.666                        & 0.582                        & 0.807                        & 0.717                        & 0.552                        & 0.517                        & 0.578                        & 0.555                        & 0.664                        & 0.594                        &                          &                          \\
\multirow{-8}{*}{Weather} & \multirow{-2}{*}{720} & TEFL & {\color[HTML]{FF0000} 0.400} & {\color[HTML]{FF0000} 0.420} & {\color[HTML]{FF0000} 0.476} & {\color[HTML]{FF0000} 0.468} & {\color[HTML]{FF0000} 0.446} & {\color[HTML]{FF0000} 0.467} & {\color[HTML]{FF0000} 0.572} & {\color[HTML]{FF0000} 0.520} & {\color[HTML]{FF0000} 0.458} & {\color[HTML]{FF0000} 0.475} & {\color[HTML]{FF0000} 0.470} & {\color[HTML]{FF0000} 0.470} & \multirow{-2}{*}{29.2\%} & \multirow{-2}{*}{20.9\%} \\ \bottomrule[1.5pt]
\end{tabular}
\end{table}

\clearpage
\newpage
\end{document}